\documentclass[11pt]{article}
\usepackage{fullpage}

\usepackage{amsmath}
\usepackage{amssymb}
\usepackage{epsfig, psfrag}
\usepackage{amsthm}
\usepackage{latexsym}
\usepackage{bbm}
\usepackage{soul}
\usepackage{nicefrac,enumitem}
\usepackage{xfrac}

\usepackage[ruled,lined]{algorithm2e}
\SetKw{Int}{Initialize:}
\SetKw{Def}{Definition:}

\newcommand{\tcpgray}[1]{\texttt{\footnotesize\color{black!70}// #1}}
\newcommand{\tcpblue}[1]{\texttt{\footnotesize\color{black!60}// {\color{blue}#1}}}

\usepackage{amsfonts} 
\usepackage{dsfont} 
\usepackage{euscript} 
\usepackage{bm}

\allowdisplaybreaks
\usepackage{mathtools}

\usepackage{epstopdf}
\usepackage{subfig}
\usepackage{graphics,color}
\usepackage{graphicx}
\usepackage[dvipsnames]{xcolor}
\usepackage{hyperref}
 
\usepackage{tikz}
\usetikzlibrary{shapes.geometric, arrows}
\tikzstyle{startstop} = [rectangle, rounded corners, minimum width=7.5cm, minimum height=2.0cm,text centered, draw=black, fill=gray!25, text width=7.5cm]
\tikzstyle{startstopwide} = [rectangle, rounded corners, minimum width=8.5cm, minimum height=2.5cm,text centered, draw=black, fill=gray!25, text width=8.5cm]
 
\usepackage[utf8]{inputenc} 
\usepackage[T1]{fontenc}    
\usepackage{booktabs}       
\usepackage{lipsum}
\usepackage{microtype}      
\usepackage{natbib}

\newtheorem{remark}{Remark}
\newtheorem{definition}{Definition}
\newtheorem{lemma}{Lemma}
\newtheorem{theorem}{Theorem}

\newenvironment{proofsketch}{\textit{Proof sketch.}}{\hfill$\square$}

\def\expe{\mathbb{E}}   
\def\P{\mathsf{P}}   

\def\argmax{\mathop{\rm arg\,max}}

\newcommand{\indicate}[1]{\mathbf{1}{\left\{#1\right\}}}

\usepackage{mathtools}

\graphicspath{{./Images/}}
 \date{}

\title{\huge When to Call Your Neighbor?\\ Strategic Communication in Cooperative Stochastic Bandits}

\author{%
	Udari Madhushani\thanks{Princeton University, 41 Olden Street, Princeton, NJ 08544,
		\texttt{email:(udarim, naomi) @princeton.edu}} \qquad    
	Naomi Ehrich Leonard\footnotemark[1]\qquad
}

\begin{document}

\maketitle

\begin{abstract}
In cooperative bandits, a framework that captures essential features of collective sequential decision making, agents can minimize group regret, and thereby improve performance, by leveraging shared information. However, sharing information can be costly, which motivates developing policies that minimize group regret while also reducing the number of messages communicated by agents. Existing cooperative bandit algorithms obtain optimal performance when agents share information with their neighbors at \textit{every time step}, i.e., full communication. This requires $\Theta(T)$ number of messages, where $T$ is the time horizon of the decision making process. We propose \textit{ComEx},  a novel cost-effective communication protocol in which the group achieves the same order of performance as full communication while communicating only $O(\log T)$ number of messages. Our key step is developing a method to identify and only communicate the information crucial to achieving optimal performance. Further we propose novel algorithms for several benchmark cooperative bandit frameworks and show that our algorithms obtain \textit{state-of-the-art} performance while consistently incurring a significantly smaller communication cost than existing algorithms.
\end{abstract}


\section{Introduction}
\label{sec:intro}
Sequential decision making in uncertain environments has been extensively studied over the past several decades due to its wide range of real world applications including recommender systems, user-targeted online advertising ~\citep{tossou2016algorithms}, clinical trials ~\citep{Durand2018ContextualBF} and target  searching (e.g. finding nuclear or a temperature source) in robotics. Making optimal decisions under uncertainty requires striking a balance between exploring the environment to identify better decisions and exploiting the decisions that are already known to produce higher outcomes. In collective decision making, i.e., a group of agents making sequential decisions, performance can be greatly improved through cooperative communication by sharing information about the environment. However, often times communication is time consuming and expensive. For example, consider a recommender systems, in which multiple servers networked to handle high demands. In this case high communication between servers can lead to service latency. Similarly, for a group of robots, communication can increase battery power consumption. Thus the cost associated with communication makes it desirable to reduce the amount of shared information. Motivated by this we ask:

\textit{Can we minimize communication without sacrificing performance in sequential decision making?}

A crucial step in answering this question is, identifying which information is most valuable. We study this problem in bandit framework, which  models sequential decision making in uncertain environments ~\citep{lai1985asymptotically}. In stochastic bandits, an agent repeatedly pulls an arm from a given set of arms and receives a  reward drawn from the probability distribution associated with the arm. The goal is maximizing cumulative reward. In an uncertain environment, the agent is required to execute a combination of \textit{exploiting actions}, i.e., pulling the arms that are known to provide high rewards, and \textit{exploring actions}, i.e., pulling  lesser known arms in order to identify arms that might potentially provide higher rewards \cite{auer2002finite}. In cooperative bandits a group of agents are faced with the same bandit problem and the goal is maximizing cumulative group reward \cite{landgren2016distributed}. Agents can obtain optimal performance by sharing all information they obtained about the arms, i.e., full communication. Thus more specifically we ask how we can minimize communication while obtaining same level of performance as full communication?

In cooperative bandits it is most useful for agents to obtain information about suboptimal arms. Each agent can reduce the number of pulls drawn from suboptimal arms by leveraging communication to reduce the uncertainty associated with the estimates of suboptimal arms. Any efficient stochastic bandit algorithm pulls suboptimal arms logarithmically in time. Thus, when communication is costly, it is desirable to communicate reward values received from suboptimal arms only. Thus our problem effectively reduces to identifying when is it more likely to pull a suboptimal arm? 

We solve this problem by proposing ComEx, a new communication protocol, in which agents only communicate the rewards they receive from exploring actions. This is because exploring actions, typically lead to pulling suboptimal arms. Combining ComEx with a cooperative Upper Confidence Bound (UCB) sampling rule \cite{kolla2018collaborative}, we prove that ComEx obtains the same order of performance as full communication, while incurring a significantly smaller communication cost than full communication. We analyze performance of the algorithm using expected group cumulative regret, which is defined as the total expected loss suffered by agents due to pulling suboptimal arms. Measuring the communication cost by the number of messages shared by agents, we prove that with ComEx agents only suffer a $O(\log T)$ cost while with full communication they suffer a $\Theta(T)$ cost. 

We show that ComEx can be incorporated in a wide range of cooperative bandit algorithms to obtain same order of performance as full communication for a significantly smaller communication cost than full communication. Incorporating ComEx, we propose novel algorithms for bench mark cooperative bandit frameworks: decentralized bandits with 1.)  instantaneous rewards sharing, 2.) message passing, 3.) estimate sharing and centralized bandits with 4.) instantaneous rewards sharing 5.)  message passing. We propose another algorithm by combining ComEx with message passing and Thompson sampling. Further we provide results illustrating that our algorithms obtain \textit{state-of-the-art} performance while consistently incurring a significantly smaller communication cost than existing algorithms in these benchmark frameworks.

\noindent \textbf{Key contributions.}
We make following key contributions in this work:
\begin{itemize}[leftmargin=20pt]
    \item We propose ComEx, a novel and cost-effective communication protocol for cooperative bandits.
    \item We provide theoretical guarantees that ComEx obtains the same order group regret as full communication while incurring a $O(\log T)$ communication cost. In contrast, full communication incurs a $\Theta(T)$ communication cost.
    \item Incorporating ComEx, we propose novel algorithms in several benchmark cooperative bandit frameworks. We provide both theoretical guarantees and experimental results validating \textit{state-of-the-art} performance of our proposed algorithms. 
\end{itemize}


\section{Related work}
\label{sec:related_work}
\noindent \textbf{Decentralized reward sharing.}
In decentralized reward sharing agents share instantaneous rewards with their neighbors \cite{chakraborty2017coordinated,kolla2018collaborative,madhushani2020distributedCons,madhushani2020doesn,wang2020optimal}.
The paper \cite{kolla2018collaborative} considered that neighbors are defined according to a fixed communication graph and provide graph structure dependent regret bounds. The paper \cite{chakraborty2017coordinated,madhushani2019heterogeneous,madhushani2020dynamic,madhushani2020heterogeneous} studied the cooperative bandit problem with time varying communication structures. The papers \cite{cesa2016delay,bar2019individual,dubey2020cooperative} considered message passing communication rules where each agent initiates a message and send the message to its neighbors. A message received from a neighbor is subsequently forwarded to other neighbors. 

\noindent \textbf{Decentralized estimate sharing.}
In estimate sharing each agent share the estimated average reward and number of arm pulls from each arm with its neighbors defined according to a fixed communication graph. The paper \cite{szorenyi2013gossip} considered a P2P communication where an agent is only allowed to communicate with two other agents at each time step. The papers 
\cite{landgren2016distributed,landgren2016distributedCDC,martinez2019decentralized,landgren2020distributed} used a running consensus algorithm to update estimates and provide graph-structure-dependent performance.

\noindent \textbf{Centralized leader-follower setting.} 
 A communication strategy where agents observe the rewards and choices of their neighbors according to a leader-follower setting is considered in \cite{landgren2018social,kolla2018collaborative,wang2020optimal}. In \cite{landgren2018social,kolla2018collaborative}, followers pull the last arm pulled by their neighbors. 
 In \cite{wang2020optimal}
 one leader explores and estimates the mean reward of arms, while all other agents pull the arm with highest estimated mean per the leader. 
 
\noindent \textbf{Communication cost.}
The paper \cite{tao2019collaborative} considered a pure exploration bandit problem and measures the communication by the number of times agents communicate. The paper \cite{madhushani2020dynamic} proposed a communication protocol where agents observe their neighbors when they have high uncertainty about arms.  
\citealt{wang2020optimal} proposed a
leader-follower algorithm with a constant communication cost. The paper \cite{Wang2020Distributed} proposed an algorithm that achieves near-optimal performance where agents achieve sublinear expected regret. In their work,  communication cost is independent of time and measured by the amount of data transmitted. 
 
\noindent \textbf{Distributed Thompson sampling.} Recently \cite{verstraeten2019multi,lalitha2020bayesian} proposed distributed Thompson sampling rules. The paper  \cite{verstraeten2019multi} studied the problem with sparse communication structures. The paper  \cite{lalitha2020bayesian} provided regret guarantees that matches the corresponding centralized regret guarantees.
 

\section{ComEx: Communicate When Exploring}\label{sec:ComEx}
In this section we provide mathematical formulation and and intuition of our communication protocol.

\noindent \textbf{Notations.}
For any positive integer $M$ we denote the set $\{1,2,\ldots, M\}$ as $[M]$. We define $\indicate{x}$ as an indicator variable that takes value 1 if $x$ is true and 0 otherwise. Further, we use $X\symbol{92}x$ to denote the set $X$ excluding the element $x.$ We use $ | X |$ to denote the number of elements in set $X.$ For any general graph $G$ we define $\Bar{\chi}(G),\Bar{\gamma}(G)$ as clique covering number and dominating number respectively. We use $G_{\gamma}$ to denote the $\gamma^{\mathrm{th}}$ power graph of $G.$ Let $g(M,x)=M+\sum_{i=1}^N\left(12\log (3(x+1)) +3\log { (x+1)}\right).$

\noindent \textbf{Cooperative stochastic bandits.}
We consider the cooperative bandit problem with $K$ arms and $N$ agents. Reward distributions of each arm $k\in [K]$ is assumed to be sub-Gaussian with mean $\mu_k$ and variance proxy $\sigma_k^2.$ At each time step $t\in [T]$ each agent $i\in [N]$ pulls an arm $A_t^{(i)}$ and receives a numerical reward $X_t^{(i)}$ drawn from the probability distribution associated with the pulled arm. Without loss of generality we assume that $\mu_1\geq \mu_2\ldots \geq \mu_K$ and define $\Delta_k:=\mu_1-\mu_k, \forall k>1$ to be the expected reward gap between optimal arm, i.e., the arm with highest mean reward, and arm $k.$ Let $\Bar{\Delta}:=\min_{k\neq 1, k\in [K]}\Delta_{k}$ be the minimum expected reward gap. We make following assumptions.

{\bf{Assumptions:}}\\
\indent(\textbf{A1}) When more than one agent pulls the same arm at the same time they receive rewards independently drawn from the probability distribution associated with the pulled arm.\\
\indent(\textbf{A2}) All the agents know $\sigma^2\geq\sigma^2_k, \forall k,$ an upper bound of the variance proxy associated with arms.

\noindent \textbf{Communication over a general graph.} 
Let $G(V,E)$ be a general graph that encodes the hard communication constraints among agents. The vertex set $V$ is the set of agents $[N]$ and each edge $(i,j) \in E$ indicates that agents $i$ and $j$ are neighbors. We consider that agents directly communicate with their neighbors only. Let $\indicate{(i,i)\in E}=1,\forall i.$ At each time step $t$ we define the communication between agents by $G_t(V,E_t)$ where $E_t\subseteq E.$ Let $d^{(i)}$ be the degree of agent $i$. Let $G_{\gamma}$ denote the $\gamma^{\mathrm{th}}$ power graph of $G.$  Denote $d_{\gamma}^{(i)}$ to be the degree of agent $i$ in graph $G_{\gamma}$, i.e., number of agents within a distance of $\gamma$ from agent $i$ in graph $G.$ For any $\gamma$ let $d_{\gamma}^{(i)^{+}}=d_{\gamma}^{(i)}+1.$

We denote $\mathbf{m}^{(i)}_t$ as the message shared by agent $i$ at time $t$ with its neighbors. This can be either a single message containing information about a particular arm pull, typically the last arm pull of agent $i,$ or a concatenation of information about several arm pulls by more than one agent over several previous time steps. We define $n^{(i)}_k(t):= \sum_{\tau= 1}^t \indicate{A^{(i)}_{\tau} = k}$ and $N^{(i)}_k(t):= \sum_{\tau = 1}^t \sum_{j=1}^N \indicate{A^{(j)}_{\tau} = k} \indicate{(i,j) \in E_{\tau}}$ to be the number of times until time step $t$ that agent $i$ pulled arm $k$ and observed reward values from arm $k$, respectively. Note that the number of observations $N^{(i)}_k(t)$ is the sum of the number of pulls drawn by agent $i$ of arm $k$ and the number of times agent $i$ received reward values of arm $k$ from its neighbors. Let $\widehat{\mu}_k^{(i)}(t)$ denote agent $i$'s estimated average reward of arm $k$ at time $t.$

\noindent \textbf{Regret and communication cost.}
Following the convention we define regret as the loss suffered by agents due to pulling suboptimal arms. Let $R(t)$ be the cumulative group regret at time $t.$ Then the expected cumulative group regret can be given as $
\expe \left[R(t)\right]:= \sum_{i=1}^N\sum_{k=2}^K\Delta_k\expe[n^{(i)}_k(t)].$
We define the communication cost as the number of messages shared by agents. We consider the cost of sharing a concatenated message to be the number of single messages included in it. Let $L(t)$ be the cumulative group communication cost at time $t$. Then, the expected group communication cost can be given as 
$
\expe[L(t)]:=\sum_{i=1}^N\sum_{\tau=1}^t\expe\left[\Big | \mathbf{m}_{\tau}^{(i)}\Big |\right].
$


\noindent \textbf{Proposed communication protocol: ComEx.}
We propose ComEx, a cost-effective partial communication protocol that obtains same order of performance as full communication. 

\begin{minipage}{0.65\linewidth}
\begin{algorithm}[H] 
\caption{ComEx}
    \SetAlgoLined
    \KwIn{Bandit environment, algorithm  parameters}
    
    \For{\text{\normalfont\ each iteration} $t \in [T]$}{
            
         \For{\text{\normalfont\ each agent} $i \in [N]$}{
         \vspace{2pt}
           \tcpgray{Sampling phase}
           
           Sampling rules: Cooperative UCB,
            Cooperative Thompson
        
        \vspace{2pt}
           \tcpgray{Message generating phase}
           
           \tcpblue{Replace full communication with ComEx}
            
            \If {$A_{t}^{(i)}\neq \argmax_k{\widehat{\mu}_k^{(i)}(t-1)}$}{ 
             $\textsc{Create}\left(m_t^{(i)}:=\Big \langle i,t, A_t^{(i)}, X_t^{(i)}\Big \rangle\right)$
                
            }
        }
        \For {\text{\normalfont\ each agent} $i \in [N]$}{
        \vspace{2pt}
           \tcpgray{Communication phase}
           
           Communication rule: 
               Decentralized (or centralized) instantaneous reward sharing, 
               Decentralized (or centralized) message passing
 
            \vspace{2pt}
              \tcpgray{Estimate updating phase}
              
        }
        
        \For {\text{\normalfont\ each arm} $k \in [K]$}{
            
                \textsc{Calculate}   $\left( \widehat{\mu}_k^{(i)}(t),N_k^{(i)}(t)\right)$ 
            }
    }
\end{algorithm}
\end{minipage}
\hfill
\begin{minipage}{0.32\linewidth}
    \centering
    \includegraphics[width=0.95\linewidth]{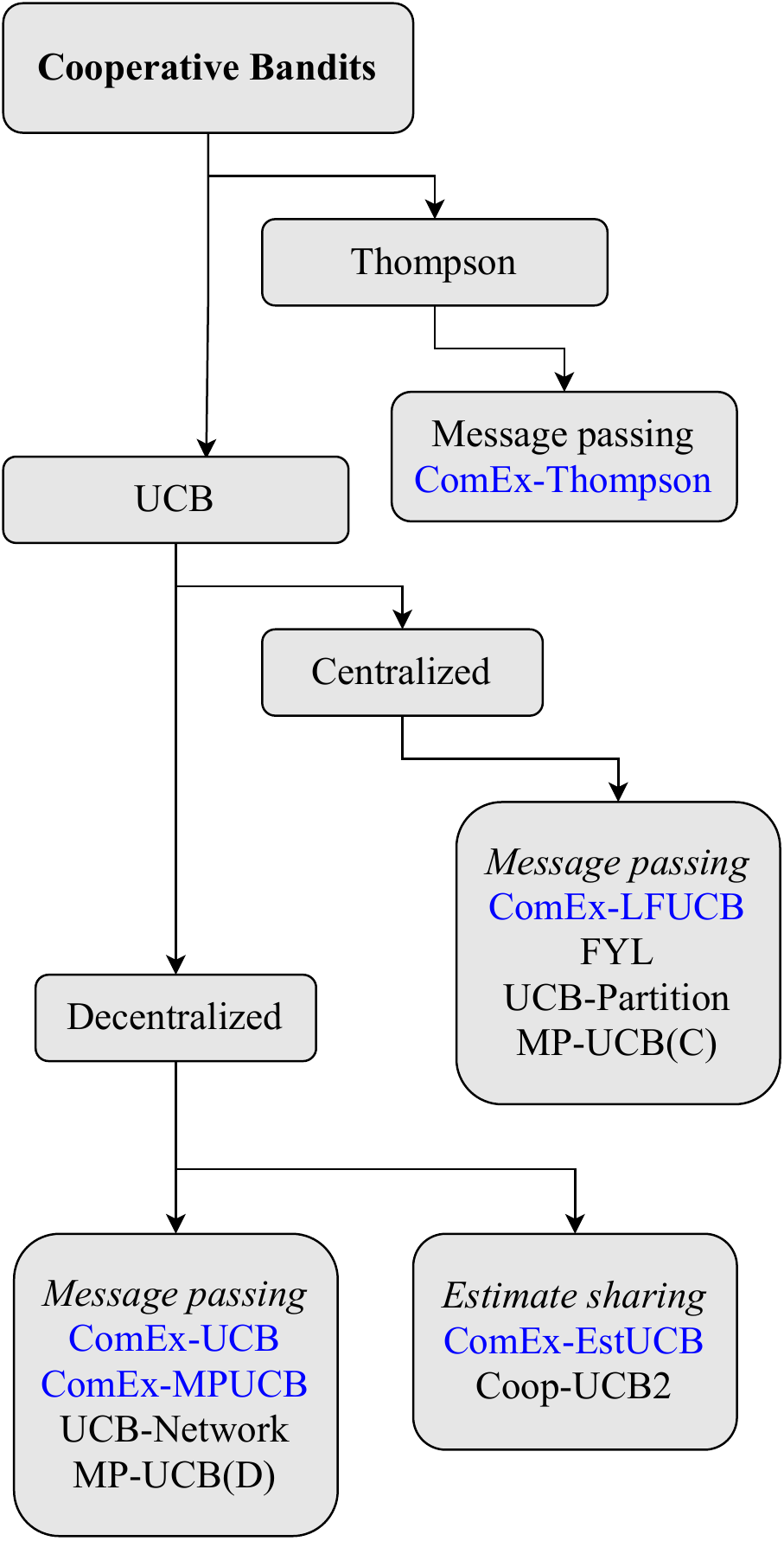}
      \captionof{figure}{A summary of our proposed algorithms and existing state-of-the-art algorithms for different cooperative bandit frameworks.}\label{fig:summary}
    
        
\end{minipage}

As motivated above, information about suboptimal arms is most valuable to agents seeking to maximize expected cumulative reward. This is because, with information from neighbors on a suboptimal arm, an agent can obtain a sufficiently accurate estimate of the expected reward of the suboptimal arm without having to pull the arm by itself. Agents typically pull suboptimal arms when they are exploring. Thus, to provide the means to maintain high performance with low communication costs, we propose a new communication protocol as follows in which agents only share information they obtained through exploring. 

\begin{definition} {\normalfont{(ComEx communication protocol)}}\label{def:comrule}
Each agent $i$ initiates sharing the message $m_t^{(i)}:=\Big\langle i, t, A_t^{(i)},  X_{t}^{(i)}\Big\rangle$ if $A_t^{(i)}\neq \argmax_{k\in [K]} \widehat{\mu}_k^{(i)}(t-1)$
\end{definition}

Note that according to the above communication protocol agents initiate sharing messages only about the rewards received from the arms that are instantaneously suboptimal i.e., arm that does not have the maximum estimated expected reward. This maximizes the chance of sharing information about suboptimal arms.

\noindent \textbf{Generalizability of ComEx.} As we will demonstrate in next few sections, our communication protocol is an easily implementable general communication protocol that can be incorporated in a wide range of cooperative bandit algorithms. We illustrate the generality by proposing novel algorithms incorporating ComEx in several cooperative bandit frameworks. Figure \ref{fig:summary} provides a summary of our algorithms and state-of-the-art algorithms in several benchmark cooperative bandit frameworks.

\section{Decentralized Cooperative Bandits}
\label{sec:dec}
In this section we propose novel algorithms for decentralized cooperative bandits.
\subsection{Decentralized instantaneous reward sharing UCB}
We present our first algorithm ComEx-UCB by combining the above communication protocol with instantaneous reward sharing. Each agent follows a sampling rule that balances exploiting with exploring. We use a natural extension of Upper Confidence Bound (UCB) algorithm as a sampling rule. In UCB at each time step $t$ for each arm $k$ each agent $i$ constructs an upper confidence bound, i.e., the sum of its estimated expected reward (empirical average of the observed rewards) and the uncertainty associated with the estimate $
C_k^{(i)}(t):=\sigma\sqrt{\frac{2(\xi+1)\log t}{N_k^{(i)}(t)}}$ where $\xi>1$, and pull the arm with highest bound. If the pulled arm is instantaneously suboptimal, the agent sends a message $m_t^{(i)}:=\Big\langle A_t^{(i)},X_t^{(i)}\Big\rangle$ to its neighbors (see Definition \ref{def:comrule}). Note that under this communication rule agents do not share concatenated messages. Thus  passing information about time step and agent id is redundant. Pseudo code for ComEx-UCB is given in Appendix \ref{sec:algUCB}.

\begin{theorem}\label{thm:regretB}{{\normalfont{(Group regret of ComEx-UCB)}}} Consider a group of $N$ agents following ComEx-UCB  while sharing instantaneous rewards over a general communication graph $G.$ Then for any $\xi\geq1.1$ expected cumulative group regret satisfies:
\begin{align*}
 \expe \left[R(T)\right] \leq\sum_{k=2}^K & \frac{8(\xi+1)\sigma}{\Delta_k}\Bar{\chi}(G)\log T 
 + \sum_{k=2}^K\Delta_k g(4N,d^{(i)})
\end{align*}
\end{theorem}
\begin{proofsketch} We follow an approach similar to the standard UCB analysis \cite{auer2002finite,dubey2020cooperative} with a few key modifications. We partition the communication graph into a set of non overlapping cliques and analyze the regret of each clique and take the summation over cliques to obtain the regret of the group. When agents are using full communication group regret can be given as the summation of a $\log T$ term that scales with the clique covering number $\Bar{\chi}(G)$ and a term, which is independent of $T.$ The second term depends on the summation of tail probabilities of arms, i.e., $\P\left(\Big | \widehat{\mu}_k^{(i)}(t)-\mu_k\Big |\geq C_k^{(i)}(t)\right).$ For full communication a similar result can be found in \citep{dubey2020cooperative}. Note that full communication is a deterministic communication protocol and ComEx-UCB is a stochastic communication protocol that depends on the decision making process. Two major technical challenges in proving the regret bound for ComEx-UCB are 1.) deriving a tail probability bound for the case in which the communication between agents are stochastic and 2.) bounding the additional regret incurred by not sharing information when pulling the arm with highest estimated average reward, i.e., $A_t^{(i)}= \argmax_{k\in [K]} \widehat{\mu}_k^{(i)}(t-1)$. We overcome the first challenge by noticing that communication random variables $\indicate{(i,j)\in E_{t}}, \forall i,j, t$ are previsible, i.e., measurable with respect to the sigma algebra generated by information obtained up to time $t-1.$ We address the second challenge by proving that the number of times agents do not share information about any suboptimal arm $k$ can be bounded by tail probabilities of arm $k$ and the optimal arm.
A complete proof of Theorem \ref{thm:regretB} is given in Appendix \ref{sec:regretB}.
\end{proofsketch}

\begin{remark}
By replacing ComEx with full communication in ComEx-UCB algorithm agents obtain an expected cumulative group regret of $ \expe \left[R(T)\right]= O\left(K\Bar{\chi}(G)\log T+KN\right)$ (Appendix H
). Thus from Theorem \ref{thm:regretB} we see that ComEx obtains the same order of performance as full communication.
\end{remark}

Recall that expected communication cost under full communication is $\Theta(T).$ Now we prove that expected communication cost under ComEx is logarithmic in time. 
In ComEx-UCB algorithm agents are only sending single messages (not concatenated). Thus expected group communication cost at time step $t$ can be given as $
\expe\left[L(t)\right]=\sum_{i=1}^N\sum_{\tau=1}^T \P\left(A_{\tau}^{(i)}\neq \argmax_{k\in [K]}\widehat{\mu}_{k}^{(i)}(\tau-1)\right).
$

\begin{theorem}{{\normalfont{(Communication cost of ComEx-UCB)}}}\label{thm:comcostB} Consider a group of $N$ agents following ComEx-UCB while sharing instantaneous rewards over a general communication graph $G.$ Then for any $\xi\geq1.1$ expected group communication cost satisfies:
\begin{align*}
\expe \left[L(T)\right]\leq & 8\sigma(\xi+1)\left[\frac{N}{\Bar{\Delta}^2}+\sum_{k=2}^K\frac{\Bar{\chi}(G)}{\Delta_k^2}\right]\log T +Kg\left(7N, d^{(i)}\right)
\end{align*}
\end{theorem}

\begin{proofsketch}
Note that expected group communication cost is the sum of 1.) expected number of times agents pull any suboptimal arm when it is instantaneously suboptimal and 2.) expected number of times agents pull the optimal arm when it is instantaneously suboptimal. We note that the first term can be directly bounded by the expected number of times agents pull suboptimal arms. We prove that the second term can be bounded logarithmically in time. A detailed proof of Theorem \ref{thm:comcostB} is given in Appendix \ref{sec:comcostB}.   
\end{proofsketch}


\subsection{Decentralized message passing UCB}
We propose ComEx-MPUCB an improved version of ComEx-UCB by incorporating a message passing method \cite{suomela2013survey,bar2019individual,dubey2020cooperative} that allows agents to share the messages they initiated with agents who are within a distance of $\gamma.$ We call $\gamma$ \textit{communication density parameter.} We consider that at  time $t$ each agent $i$ initiates a message $m_t^{(i)}:=\Big\langle i, t, A_t^{(i)},  X_{t}^{(i)}\Big\rangle$ according to ComEx given in Definition \ref{def:comrule} and sends the messages to its neighbors. Subsequently the agents who receive the message forward it to their neighbors. Messages received at time $t$ are forwarded to neighbors at time $t+1$ resulting that each hop adds a delay of 1 time step. Under this message passing method $\gamma$-hop neighbors receive the message after a delay of $\gamma$ time steps. Agents do not forward the messages that are older than $\gamma-1$ and discard the messages that are older than $\gamma.$ Note that for a connected graph maximum number of time step required to pass a message between any two agents equals to the diameter of the graph. Thus we choose $\gamma$ to be an integer constant which is at most diameter of the communication graph $G.$ The pseudo code for ComEx-MPUCB is given in Appendix  \ref{sec:algMPUCB}.

\begin{theorem}{{\normalfont{(Group regret of ComEx-MPUCB)}}}\label{thm:regretMP} Consider a group of $N$ agents following ComEx-MPUCB. Then for any $\xi\geq 1.1$ expected cumulative group regret satisfies: 
\begin{align*}
 \expe &\left[R(T)\right] \leq\sum_{k=2}^K\frac{8(\xi+1)\sigma}{\Delta_k}\Bar{\chi}(G_{\gamma})\log T +\sum_{k=2}^K\Delta_k\left[(N-\mathcal{X}(G_{\gamma}))(\gamma-1)+g\left(4N,d^{(i)}_{\gamma}\right)\right]
\end{align*}
\end{theorem}
\begin{proofsketch} We see that regret under ComEx-MPUCB can be given as the summation of regret of ComEx-UCB when communication graph is $G_{\gamma}$ and the regret incurred by the delay in passing messages to agents who are not 1-hop neighbors. We prove that the expected regret due to delay is at most $(N-\Bar{\chi}(G_{\gamma}))(\gamma-1).$ A detailed proof is provided in Appendix  \ref{sec:regretMP}.
\end{proofsketch}

\begin{remark}
Similar to ComEx-UCB by replacing ComEx with full communication in ComEx-MPUCB algorithm agents obtain an expected cumulative group regret of $ \expe \left[R(T)\right]= O\left(K\Bar{\chi}(G_{\gamma})\log T+KN\right)$ (Appendix H
). Thus from Theorem \ref{thm:regretMP} we see that ComEx obtains the same order of performance as full communication.
\end{remark}

Now we proceed to prove that expected group communication cost under ComEx-MPUCB is logarithmic in time.  

\begin{theorem}{{\normalfont{(Communication cost of ComEx-MPUCB)}}} Consider a group of $N$ agents following ComEx-MPUCB with communication density parameter $\gamma.$ Then for any $\xi\geq 1.1$ expected group communication cost satisfies: \label{thm:comcostMP}
\vspace{-5pt}
\begin{align*}
\expe \left[L(T)\right] \leq \left[8(\xi+1)\sigma\left[\frac{N}{\Bar{\Delta}^2}+\sum_{k=2}^K\frac{\Bar{\chi}(G_{\gamma})}{\Delta_k^2}\right]\log T +K\left[(N-\Bar{\chi}(G_{\gamma})(\gamma-1)\right]\right]\sum_{i=1}^Nd_{\gamma-1}^{(i)^{+}}
\\
+K\sum_{i=1}^Nd_{\gamma-1}^{(i)^{+}}\cdot g\left(7N,d^{(i)}_{\gamma}\right)
\end{align*}
\end{theorem}
\vspace{-5pt}
\begin{proofsketch} Note that under ComEx-MPUCB agents send concatenated messages to their neighbors. Recall that agents do not forward the messages that are older than $\gamma-1$. Thus each message initiated by agent $i$ is subsequently forwarded by all agents who are within distance of $\gamma-1$ in graph $G$. Thus we have $
 \expe\left[L(t)\right]\leq\sum_{i=1}^N\!d_{\gamma-1}^{(i)^{+}}\!\sum_{\tau=1}^t \P\left(A_{\tau}^{(i)}\neq \argmax_{k\in [K]}\widehat{\mu}_{k}^{(i)}(\tau-1)\right).   
$ A detailed proof can be found in Appendix  \ref{sec:comcostMP}.
\end{proofsketch}

\section{Centralized Cooperative Bandits}
\label{sec:cen}
We propose ComEx-LFUCB by combining ComeEx communication protocol with a leader-follower method \cite{kolla2018collaborative,landgren2018social,dubey2020cooperative,wang2020optimal}. ComEx-LFUCB provides better performance compared to its decentralized counter part ComEx-MPUCB. Let $V_{\gamma}^{\prime}$ be the set of vertices in minimal dominating set of graph $G_{\gamma}.$ We consider each agent $i\in V_{\gamma}^{\prime}$ to be a leader and all the other agents to be followers. Note that every follower has at least one leader as a neighbor. We consider that each leader uses ComEx-MPUCB and each follower copies the last action observed from its leader. For each follower $j$ a leader $i$ is assigned such that $d(i,j)=\min_{i^{\prime}}d(i^{\prime},j)$ where $d(i,j)$ is the distance between agent $i$ and agent $j$ in graph $G.$ Let $\mathcal{N}_{\gamma}^{i}$ be the set of follower of leader $i.$ We consider that each leader sends a message containing the id of the arm it pulls and whether it is instantaneously suboptimal, i.e. for $i\in V^{\prime }_{\gamma}$ at time step $t$, $m^{(i)}_t:=\Big \langle i,t, A_t^{(i)},\indicate{A_i^{(t)}\neq \argmax_{k\in [K]}\widehat{\mu}_k^{(i)}(t-1)}\Big\rangle $ to its neighbors and they subsequently forward it to their neighbors. Note that at time step $t$ follower $j\in\mathcal{N}_{\gamma}^{(i)}$ pulls the arm $A_{t-d(i,j)}^{(i)}.$ Each follower pass a message containing information about the reward and arm id if it pulls an arm that is specified as instantaneously suboptimal by its leader. Thus the followers communicate according to ComEx by initiating a message as follows. Follower $j\in \mathcal{N}^{(i)}_{\gamma}$ initiates a message $m_t^{(j)}:=\Big \langle j,t, A_t^{(j)},X_t^{(j)}\Big \rangle $ if $A_{t-d(i,j)}^{(i)}\neq  \argmax_{k\in [K]}\widehat{\mu}_k^{(i)}(t-d(i,j)-1).$ Accordingly under full communication followers share their rewards and arm pulls at every time step. Pseudo code for ComEx-LFUCB is provided in Appendix \ref{sec:algLFUCB}.

\begin{theorem}\label{thm:regretLFUCB}{{\normalfont{(Group regret of ComEx-LFUCB)}}} Consider a group of $N$ agents following ComEx-LFUCB with communication density parameter $\gamma.$ Then for any $\xi\geq 1.1$ expected cumulative group regret satisfies: 
\begin{align*}
 &\expe \left[R(T)\right] \leq\sum_{k=2}^K\frac{8(\xi+1)\sigma}{\Delta_k}\Bar{\gamma}(G_{\gamma})\log T+\sum_{k=2}^K\Delta_k\left[(N-\Bar{\gamma}(G_{\gamma}))(3\gamma-1)+\Bar{\gamma}(G_{\gamma})\cdot g(4N,d^{(i)}_{\gamma})\right]
\end{align*}
\end{theorem}

\begin{proofsketch}
We follow a similar approach to the proof of Theorem \ref{thm:regretMP} with a few key modifications followed by the argument below. Note that number of suboptimal arm pulls by each $j\in \mathcal{N}^{(i)}_{\gamma}$ can be upper bounded using  suboptimal arm pulls by $i$ and message passing delay. Note that message passing delay can be upper bounded by $d(i,j).$ A detailed proof of Theorem \ref{thm:regretLFUCB} is given in Appendix  \ref{sec:regretLFUCB}.
\end{proofsketch}

\begin{remark}
Similar to ComEx-MPUCB by replacing ComEx with full communication in ComEx-LFUCB algorithm, i.e. allowing followers to share information about arm pulls at every time step, agents obtain an expected cumulative group regret of $ \expe \left[R(T)\right]= O\left(K\Bar{\gamma}(G_{\gamma})\log T+KN\right)$ (Appendix H
). Thus from Theorem \ref{thm:regretLFUCB} we see that ComEx obtains the same order of performance as full communication.
\end{remark}

Now we provide theoretical guarantees that expected group communication cost under ComEx-LFUCB is logarithmically bounded in time. 

\begin{theorem}{{\normalfont{(Communication cost of ComEx-LFUCB)}}} Consider a group of $N$ agents following ComEx-LFUCB with communication density parameter $\gamma.$ Then for any $\xi\geq 1.1$ expected group communication cost satisfies:
\vspace{-10pt}
\label{thm:comcostLF}
\begin{align*}
\expe \left[L(T)\right]\leq 
\left[8(\xi+1)\sigma\left[\frac{N}{\Bar{\Delta}^2}+\sum_{k=2}^K\frac{\Bar{\gamma}(G_{\gamma})}{\Delta_k^2}\right]\log T +K\left[(N-3\Bar{\gamma}(G_{\gamma})(\gamma-1)\right]\right]\sum_{i=1}^Nd_{\gamma-1}^{(i)^{+}}
\\
+K\sum_{i=1}^Nd_{\gamma-1}^{(i)^{+}}\cdot\Bar{\gamma}(G_{\gamma}\cdot g\left(7N,d^{(i)}_{\gamma}\right)
\end{align*}
\end{theorem}
\vspace{-10pt}
\begin{proofsketch}
Note that the expected number of times a leader initiates a message can be upper bounded by twice the expected number of its suboptimal arm pulls. Further the number of times each follower $j\in\mathcal{N}_{\gamma}^{(i)}$ initiates a message can be bounded by the number of instantaneously suboptimal arms pulled by the leader $i$. Similar to ComEx-MPUCB in ComEx-LFUCB agents send concatenated messages to their neighbors. Thus each message initiated by any agent $i$ is subsequently forwarded by all agents who are within distance of $\gamma-1$ in graph $G$. A detailed proof can be found in Appendix  \ref{sec:comcostLFUCB}.
\end{proofsketch}

\begin{remark}
Algorithm and results provided in this Section can be specialized to centralized cooperative bandits with instantaneous reward sharing by substituting $\gamma=1.$
\end{remark}

\begin{remark}{\normalfont (Upper bound on communication cost)} Although smaller $\Delta_k$ values lead to larger upper bounds for each algorithm (with communication density $\gamma$) presented in Section \ref{sec:dec} and \ref{sec:cen} communication cost is upper bounded by  $T\sum_{i=1}^Nd^{(i)^{+}}_{\gamma-1}.$

\end{remark}


\section{Additional Algorithms}
\label{sec:add_alg}
We propose two more algorithms, thus extending ComEx to additional cooperative bandit frameworks. We leave providing theoretical guarantees for these as future work.

\noindent \textbf{Estimate sharing.}
We propose ComEx-EstUCB by combining ComEx with estimate sharing \cite{landgren2016distributedCDC,martinez2019decentralized,landgren2020distributed}, which obtains better performance than instantaneous reward sharing. In estimate sharing, for each arm $k,$ agents maintain estimated sum of rewards and estimated number of pulls from the arm. At each time step, agents average their estimates with their neighbors according to a consensus protocol and update the estimates by incorporating the information of arm pull at that time step. We refer readers to \citealt{landgren2020distributed} for more details. In ComEx-EstUCB agents only average estimates of instantaneously sub optimal arms. Pseudo code for ComEx-EstUCB is given in Appendix \ref{sec:algEstUCB}. 


\noindent \textbf{Thompson sampling.} 
We extend our communication protocol to cooperative Thmpson bandits as follows. We propose ComEx-MPThompson, a new algorithm by replacing UCB sampling rule with Thompson sampling rule in ComEx-MPUCB as follows. We combine ComEx with message passing and a natural extension of Thompson sampling to cooperative bandits. Here we provide a brief description of cooperative Thompson sampling rule and refer readers to \cite{lalitha2020bayesian} for more details. Algorithm is initialized by each agent assigning a suitable prior distribution to each arm. Typically Gaussian priors are used for Gaussian reward distributions and Beta priors are used for Bernoulli distributions. At each time step each agent constructs a posterior distribution for each arm using prior distribution and available reward information at that time step. Each agent draws a sample from posterior distributions associated with each arm and pull the arm with highest sampled value. Agents initialize messages according to ComEx and pass the messages to neighbors using a similar protocol given in ComEx-MPUCB. Pseudo code for ComEx-MPThompson is given in Appendix \ref{sec:algMPTHmp}. 


\section{Experimental Results}
\label{sec:experiments}
In this section we provide numerical simulations illustrating our results and validating our theoretical claims. All the experiments were run on the first author's personal laptop. We show that ComEx obtains same order of performance, i.e., same order of group regret, as full communication for a significantly smaller communication cost than full communication. We also demonstrate that our algorithms outperform state-of-the-art algorithms in several bandit frameworks. 

\begin{figure*}[h]
    \centering
    \includegraphics[width=0.98\textwidth]{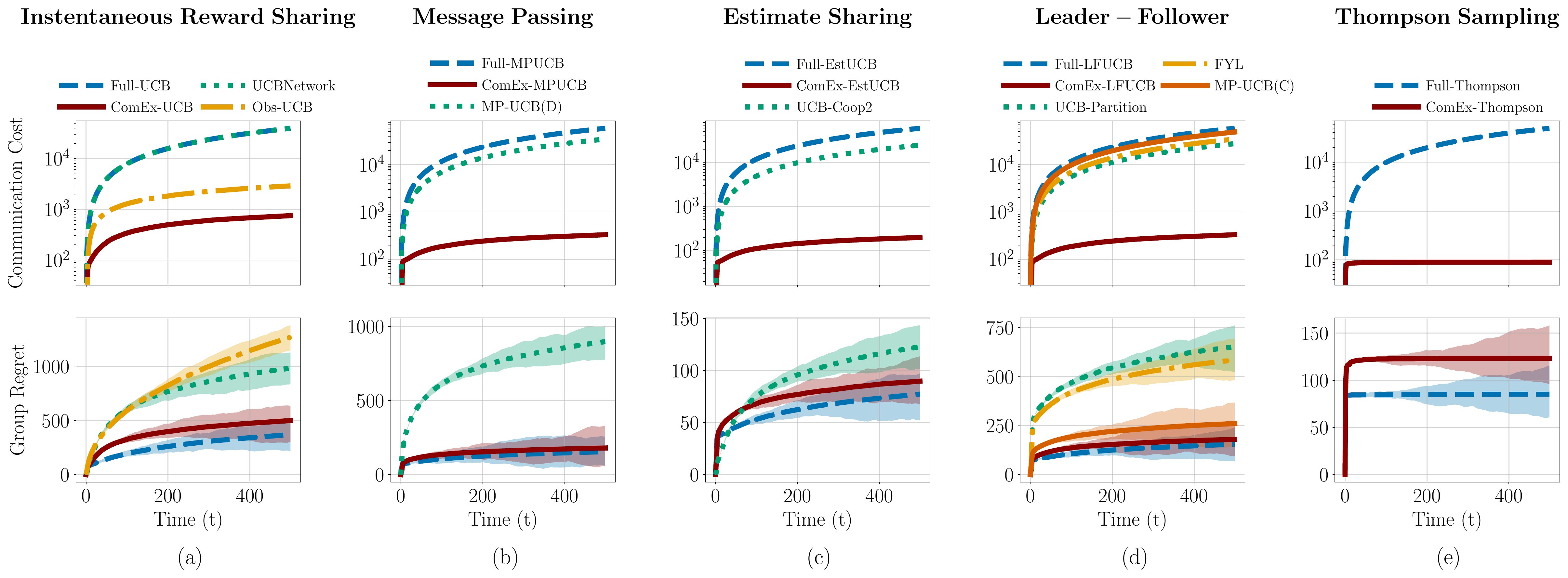}
    \caption{\small{ A comparison of expected cumulative group regret and communication cost of our algorithms and existing state-of-the-art algorithms in several benchmark cooperative bandit frameworks.}}
    \label{fig:performance}
\end{figure*}
\vspace{1 em}
\noindent \textbf{Experimental setup.} We provide simulation results for following cooperative bandit frameworks 1) decentralized instantaneous reward sharing, 2) decentralized message passing, 3) decentralized estimate sharing, 4) centralized leader-follower, and 5) Thompson sampling. We compare performance of our algorithms (ComEx-UCB, ComEx-MPUCB, ComEx-EstUCB, ComEx-LFUCB and ComEx-Thompson) with their corresponding full communication algorithms (Full-UCB, Full-MPUCB, Full-EstUCB, Full-LFUCB and Full-Thompson) and state-of-the art algorithms in each framework. For all simulations presented in this section we consider 10 arms $(K=10),$ 100 agents $(N=100)$ and 500 time steps $(T=500).$ Communication graph between agents is considered to be a Erdos Renyi random graph with edge probability $0.7.$ Results are averaged over 100 Monte Carlo simulations. Additional experimental results for different graph structures and parameters ($\xi,\gamma$) are provided in Appendix  \ref{sec:simulation}.

\paragraph{Hyper parameters}
We use tuning parameter $\xi=1.01$ for UCB based algorithms. For results provided in Figure \ref{fig:performance}(b)-\ref{fig:performance}(e) we use communication density parameter $\gamma=5.$ None of the competing algorithms, except  UCB-Coop2, MP-UCB(D) and MP-UCB(C) have hyperparameters. We tuned parameters of UCB-Coop2 to get best results of that algorithm and used $\kappa=0.02,\gamma^{\prime}=1.001, \eta=0.001$ (Equations 9 and 15 in \cite{landgren2020distributed}. Here we $\gamma^{\prime}$ to avoid confusing with communication parameter $\gamma$ used in this paper) for final results. Decreasing $\gamma^{\prime}$ below 1.001 and $\eta$ below 0.001 did not offer any significant improvement. MP-UCB(D) and MP-UCB(C) are originally proposed in \cite{dubey2020cooperative} for heavy-tailed distributions, and we adapt them to sub-Gaussian distributions as directed by the authors. For MP-UCB(D) and MP-UCB(C) we considered the same $C_k^{(i)}(t)$ as in our algorithms. Thus we used the same $\xi=1.01$ value for a fair comparison. 

For results provided in Figures \ref{fig:performance}(a) and \ref{fig:performance}(d), we consider reward distributions to be bounded $[0,1].$ We consider triangle distributions with mod 1 for the optimal arm and mod 0 for all sub-optimal arms. In simulations provided in Figures \ref{fig:performance}(b), \ref{fig:performance}(c) and \ref{fig:performance}(e) we consider Gaussian reward distributions. Expected reward for the optimal arm  is $\mu_{1}=11$ and for all sub-optimal arms $k>1$ is $\mu_{k}=10$. We let variance associated with all arms be $\sigma_k^2 =1,\forall k$. We use the notation Obs-UCB to denote the algorithm presented in \cite{madhushani2020dynamic}.

\textbf{ComEx obtains same order of performance as full communication.}
Our results in Figure~\ref{fig:performance} illustrate that ComEx obtains the same order of performance, i.e., same order of group regret, as full communication. From Comparing Figures \ref{fig:performance}(a) and \ref{fig:performance}(b) we see that performance difference between full communication and ComEx decrease when communication density $\gamma$ increase. Comparing Figure \ref{fig:performance}(e) with others we see that performance difference between full communication and ComEx is smaller when agents are using UCB based sampling rules and Thompson based sampling rules. All results illustrate that our algorithms consistently out preforms state-of-the-art algorithms in all five benchmark cooperative bandit frameworks.

\textbf{ComEx only incurs a logarithmic communication cost.} Our simulation results also illustrate that ComEx only incurs a logarithmic communication cost. In Figure \ref{fig:performance}(a) we observe that Obs-UCB also incurs a logarithmic cost. However ComEx-UCB incurs a smaller cost than Obs-UCB while suffering a smaller group regret. Further, results illustrate that ComEx enabled algorithms incurs a significantly smaller communication cost compared to existing state-of-the-art algorithms.

\paragraph{Additional discussion.} State-of-the-art algorithm for leader-follower setting is DPE2 in \cite{wang2020optimal}. DPE2 uses a phased communication protocol, where during the leader selection phase, which lasts at least $2D$ rounds, where $D$ is the diameter of the graph, agents do not pull arms. Thus, this phase accumulates an expected group regret of at least $2DN\mu_1.$ In our experimental setup, this alone exceeds the regret accumulated by our algorithms during the entire time horizon. So a meaningful comparison cannot be provided without modifying DPE2 to allow pulling arms during the leader selection phase.


\section{Discussion}
\label{sec:discussion}
\paragraph{Limitations.}
Main limitation of this work is that all the theoretical claims are provided using upper bounds. Obtaining lower bounds for cooperative bandits that communicate over general graphs are difficult due to the complex nature of the probability distribution associated with the sampling process of agents. This is an active area of research. We provide a discussion in Appendix B for the optimality of our regret bounds by providing a lower bound when $G$ is a complete graph. 

\paragraph{Future extensions.} We plan to analyse regret and communication cost for the algorithms provided in Section \ref{sec:add_alg}. Our intuition can be extended to the collision setting by not allowing agents to share information about the first $N$ instantaneously optimal arms. In the collision setting when more than one agent pulls the same arm at the same time step a collision occurs. This causes agents to either split the reward or completely loose the reward at that time step. Another extension will be proposing similar algorithms for linear bandits and adversarial bandits. 


\section{Conclusion}
\label{sec:cocnl}
We proposed ComEx, a general and effective communication protocol which obtains same order of performance as full communication but incurs significantly smaller communication cost than the latter. Next, we proposed novel algorithms for several benchmark bandit frameworks by incorporating ComEx protocol. We provided theoretical guarantees followed by experimental results illustrating the \textit{state-of-the-art} performance of our algorithms. 

\clearpage

\bibliography{comEx}
\bibliographystyle{apsr}
\clearpage


\clearpage

\appendix


\section{Proof of Theorem \ref{thm:regretB}} \label{sec:regretB}
We begin the proof of Theorem \ref{thm:regretB} by proving a few useful lemmas.

\begin{lemma} {\bf{(Restatement of results from \citep{auer2002finite})}}\label{lem:tailrestate}
Let $\eta_k=\left(\frac{8(\xi+1)\sigma^2}{\Delta^2_k}\right)\log T.$ For any suboptimal arm $k$ and $\forall i,t$ we have
\begin{align*}
 \P\left(A^{(i)}_{t+1} = k, N_k^{(i)}(t)> \eta_k\right)
    \leq 
    \P \left(\widehat{\mu}_{1}^{(i)}(t)\leq \mu_{1}-C_{1}^{(i)}(t)\right)+\P \left(\widehat{\mu}_{k}^{(i)}(t)\geq \mu_{k}+C_{k}^{(i)}(t)\right)
\end{align*}
\end{lemma}

\begin{proof}
Note that for any $k>1$ we have
\begin{align*}
    \left \{A^{(i)}_{t+1} = k\right \}
    &\subset 
    \left \{Q^{(i)}_k(t) \geq Q^{(i)}_{1}(t)\right \}
    \\
    &\subset 
    \left\{\left\{\mu_{1}<\mu_{k}+2C_{k}^{(i)}(t)\right\}\cup\left\{\widehat{\mu}_{1}^{(i)}(t)\leq \mu_{1}-C_{1}^{(i)}(t)\right\}\cup \left\{\widehat{\mu}_{k}^{(i)}(t)\geq \mu_{k}+C_{k}^{(i)}(t)\right\}\right\}.
\end{align*}
Let $\eta_k=\left(\frac{8(\xi+1)\sigma^2}{\Delta^2_k}\right)\log T$. Since $N_k^{(i)}(t)> \eta_k$ the event $\left\{\mu_{1}<\mu_{k}+2C_{k}^{(i)}(t)\right\}$ does not occur. Thus we have
\begin{align*}
    \P\left(A^{(i)}_{t+1} = k, N_k^{(i)}(t)> \eta_k\right)
    \leq 
    \P \left(\widehat{\mu}_{1}^{(i)}(t)\leq \mu_{1}-C_{1}^{(i)}(t)\right)+\P \left(\widehat{\mu}_{k}^{(i)}(t)\geq \mu_{k}+C_{k}^{(i)}(t)\right)
\end{align*}
This concludes the proof of Lemma \ref{lem:tailrestate}.
\end{proof}

\begin{lemma}\label{lem:RegDecomp}
Let $\Bar{\chi}(G)$ is the clique covering number of graph $G.$ Let $\eta_k=\left(\frac{8(\xi+1)\sigma^2}{\Delta^2_k}\right)\log T.$ Then we have
\begin{align*}
    \sum_{i=1}^N\expe[n^{(i)}_k(T)]  \leq \Bar{\chi}(G) \eta_k+N+\sum_{i=1}^N\sum_{t=1}^{T-1}\left[\P \left(\widehat{\mu}_{1}^{(i)}(t)\leq \mu_{1}-C_{1}^{(i)}(t)\right)+\P \left(\widehat{\mu}_{k}^{(i)}(t)\geq \mu_{k}+C_{k}^{(i)}(t)\right)\right]
\end{align*}
\end{lemma}
\begin{proof}
Let $\EuScript{C}$ be a non overlapping clique covering of $G$. Note that for each suboptimal arm $k > 1$ we have
\begin{align}
    \sum_{i=1}^N\expe[n^{(i)}_k(T)]
    & = 
  \sum_{i=1}^N \sum_{t=1}^T\P\left(A^{(i)}_t = k\right)=\sum_{\mathcal{C}\in \EuScript{C}}\sum_{i\in \mathcal{C}}\sum_{t=1}^T \P\left(A^{(i)}_t = k\right) \label{eq:rearrange}
\end{align}

 Let $\tau_{k,\mathcal{C}}$ be the maximum time step such that the total number of pulls from arm $k$ shared by agents in the clique $\mathcal{C}$ is at most $\eta_k.$ This can be stated as 

$
\tau_{k,\mathcal{C}} := \max \left\{t\in [T] :  \sum_{i\in \mathcal{C}}\sum_{\tau=1}^t\indicate{A_{\tau}^{(i)}=k,A_{\tau}^{(i)}\neq \argmax_{l\in [K]}\widehat{\mu}_{l}^{(i)}(\tau-1)}\leq \eta_k\right\}
$. Then for all $i\in \mathcal{C}$ we have $
N_k^{(i)}(t)> \eta_k, \forall t> \tau_{k,\mathcal{C}}.$ 
We analyse the expected number of times all agents pull suboptimal arm $k$ as follows.
\begin{align}
\sum_{\mathcal{C}\in \EuScript{C}}\sum_{i\in \mathcal{C}}\sum_{t=1}^T \indicate{A^{(i)}_t = k}  &= \sum_{\mathcal{C}\in \EuScript{C}}\sum_{i\in \mathcal{C}}\sum_{t=1}^{\tau_{k,\mathcal{C}}} \indicate{A^{(i)}_t = k}
\\
& +  \sum_{\mathcal{C}\in \EuScript{C}}\sum_{i\in \mathcal{C}}\sum_{t>\tau_{k,\mathcal{C}}}^{T} \indicate{A^{(i)}_{t} = k, N_k^{(i)}(t-1) >\eta_k} \label{eq:regdec}
\end{align}
Taking the expectation of \eqref{eq:regdec} we have
\begin{align}
\sum_{\mathcal{C}\in \EuScript{C}}\sum_{i\in \mathcal{C}}\sum_{t=1}^T \P\left(A^{(i)}_t = k\right)  
& =   \sum_{\mathcal{C}\in \EuScript{C}}\sum_{i\in \mathcal{C}}\sum_{t=1}^{\tau_{k,\mathcal{C}}} \P\left(A^{(i)}_t = k\right)
\\
& +  \sum_{\mathcal{C}\in \EuScript{C}}\sum_{i\in \mathcal{C}}\sum_{t>\tau_{k,\mathcal{C}}}^T \P\left(A^{(i)}_t = k,N_k^{(i)}(t-1) >\eta_k\right)\label{eq:regDecCom}
\end{align}
Now we proceed to upper bound the first term of right hand side of \eqref{eq:regdec} as follows. Note that we have
\begin{align}
\sum_{i\in \mathcal{C}}\sum_{t=1}^{\tau_{k,\mathcal{C}}} \indicate{A^{(i)}_t = k} & =  \sum_{i\in \mathcal{C}}\sum_{t=1}^{\tau_{k,\mathcal{C}}}\indicate{A_{t}^{(i)}=k,A_{t}^{(i)}\neq \argmax_{l\in [K]}\widehat{\mu}_{l}^{(i)}(t-1)}
\nonumber \\
& +\sum_{i\in \mathcal{C}}\sum_{t=1}^{\tau_{k,\mathcal{C}}}\indicate{A_{t}^{(i)}=k,A_{t}^{(i)}= \argmax_{l\in [K]}\widehat{\mu}_{l}^{(i)}(t-1)}
\nonumber \\
&\leq \eta_k + \sum_{i\in \mathcal{C}}\sum_{t=1}^{\tau_{k,\mathcal{C}}}\indicate{A_{t}^{(i)}=k,A_{t}^{(i)}= \argmax_{l\in [K]}\widehat{\mu}_{l}^{(i)}(t-1)}\label{eq:noB}
\end{align}

Taking the expectation of (\ref{eq:noB}) we have
\begin{align}
\sum_{i\in \mathcal{C}}\sum_{t=1}^{\tau_{k,\mathcal{C}}} \P\left(A^{(i)}_t = k\right)\leq \eta_k+ \sum_{i\in \mathcal{C}}\sum_{t=1}^{\tau_{k,\mathcal{C}}}\P\left(A_{t}^{(i)}=k,A_{t}^{(i)}= \argmax_{l\in [K]}\widehat{\mu}_{l}^{(i)}(t-1)\right)\label{eq:subopt}
\end{align}

Now we proceed to upper bound last term of (\ref{eq:subopt}) as follows. Note that for any suboptimal arm $k$ we have,
\begin{align}
\P&\left(A_{t+1}^{(i)}=k,A_{t+1}^{(i)}= \argmax_{l\in [K]}\widehat{\mu}_{l}^{(i)}(t)\right) 
\nonumber\\\ 
&  \leq\P\left(\widehat{\mu}_k^{(i)}(t)+C_k^{(i)}(t)\geq \widehat{\mu}_1^{(i)}(t)+C_1^{(i)}(t), \widehat{\mu}_k^{(i)}(t)\geq \widehat{\mu}_1^{(i)}(t), \:\:
\widehat{\mu}_1^{(i)}(t)\leq \mu_{1}-C_1^{(i)}(t)\right)
\nonumber\\
& +\P\left(\widehat{\mu}_k^{(i)}(t)+C_k^{(i)}(t)\geq \widehat{\mu}_1^{(i)}(t)+C_1^{(i)}(t), \widehat{\mu}_k^{(i)}(t)\geq \widehat{\mu}_1^{(i)}(t), \:\:
\widehat{\mu}_1^{(i)}(t)> \mu_{1}-C_1^{(i)}(t)\right)
\nonumber\\\ 
&  \leq \P\left(\widehat{\mu}_1^{(i)}(t)\leq \mu_{1}-C_1^{(i)}(t)\right)
\\
& +\P\left(\widehat{\mu}_k^{(i)}(t)+C_k^{(i)}(t)\geq \widehat{\mu}_1^{(i)}(t)+C_1^{(i)}(t), \widehat{\mu}_k^{(i)}(t)\geq \widehat{\mu}_1^{(i)}(t), \:\:
\widehat{\mu}_1^{(i)}(t)> \mu_{1}-C_1^{(i)}(t)\right)\label{eq:condeq}
\end{align}

Now we proceed to upper bound the last term of \eqref{eq:condeq} as follows. Note that we have
\begin{align}
\P& \left(\widehat{\mu}_k^{(i)}(t)+C_k^{(i)}(t)\geq \widehat{\mu}_1^{(i)}(t)
+C_1^{(i)}(t), \widehat{\mu}_k^{(i)}(t)\geq \widehat{\mu}_1^{(i)}(t), \:\:
\widehat{\mu}_1^{(i)}(t)> \mu_{1}-C_1^{(i)}(t)\right)
\nonumber \\
& \leq\P\left(\widehat{\mu}_k^{(i)}(t)+C_k^{(i)}(t)>\mu_{1}-C_1^{(i)}(t)+C_k^{(i)}(t)\right)
\nonumber \\
& \leq \P\left(\widehat{\mu}_k^{(i)}(t)\geq \mu_k+C_k^{(i)}(t)\right).\label{eq:optsam}
\end{align}

From \eqref{eq:condeq} and \eqref{eq:optsam} we have
\begin{align}
\P\left(A_{t}^{(i)}=k,A_{t}^{(i)}= \argmax_{l\in [K]}\widehat{\mu}_{l}^{(i)}(t-1)\right)  \leq \P\left(\widehat{\mu}_1^{(i)}(t-1)\leq \mu_{1}-C_1^{(i)}(t-1)\right)
\\
+\P\left(\widehat{\mu}_k^{(i)}(t-1)\geq \mu_k+C_k^{(i)}(t-1)\right) \label{eq:addtail}
\end{align}

From \eqref{eq:regDecCom}, \eqref{eq:subopt} and \eqref{eq:addtail} we have
\begin{align}
\sum_{\mathcal{C}\in \EuScript{C}}\sum_{i\in \mathcal{C}}\sum_{t=1}^T \P\left(A^{(i)}_t = k\right) \leq \sum_{\mathcal{C}\in \EuScript{C}}\eta_k+  \sum_{\mathcal{C}\in \EuScript{C}}\sum_{i\in \mathcal{C}}\sum_{t>\tau_{k,\mathcal{C}}}^T \P\left(A^{(i)}_t = k,N_k^{(i)}(t-1) >\eta_k\right)
\nonumber\\
+ \sum_{\mathcal{C}\in \EuScript{C}}\sum_{i\in \mathcal{C}}\sum_{t=1}^{\tau_{k,\mathcal{C}}}\left[\P \left(\widehat{\mu}_{1}^{(i)}(t)\leq \mu_{1}-C_{1}^{(i)}(t)\right)+\P \left(\widehat{\mu}_{k}^{(i)}(t)\geq \mu_{k}+C_{k}^{(i)}(t)\right)\right]
\nonumber \\
\leq \Bar{\chi}(G)\eta_k+N+  \sum_{i=1}^N\sum_{t=1}^{\tau_{k,\mathcal{C}}}\left[\P \left(\widehat{\mu}_{1}^{(i)}(t)\leq \mu_{1}-C_{1}^{(i)}(t)\right)+\P \left(\widehat{\mu}_{k}^{(i)}(t)\geq \mu_{k}+C_{k}^{(i)}(t)\right)\right]
\nonumber\\
+ \sum_{\mathcal{C}\in \EuScript{C}}\sum_{i\in \mathcal{C}}\sum_{t>\tau_{k,\mathcal{C}}}^{T-1} \P\left(A^{(i)}_{t+1} = k,N_k^{(i)}(t) >\eta_k\right)
\label{eq:subsampl}
\end{align}

From (\ref{eq:rearrange}), (\ref{eq:subsampl}) and Lemma \ref{lem:tailrestate} we have
\begin{align*}
    \sum_{i=1}^N\expe[n^{(i)}_k(T)]  
    &\leq \Bar{\chi}(G) \eta_k+N
    \\
    &+\sum_{i=1}^N\sum_{t=1}^{T-1}\left[\P \left(\widehat{\mu}_{1}^{(i)}(t)\leq \mu_{1}-C_{1}^{(i)}(t)\right)+\P \left(\widehat{\mu}_{k}^{(i)}(t)\geq \mu_{k}+C_{k}^{(i)}(t)\right)\right]
\end{align*}
This concludes the proof of Lemma  \ref{lem:RegDecomp}.
\end{proof}

Now we proceed to bound the tail probabilities as follows.
\begin{lemma}{\bf{(Tail probability bound)}}\label{lem:tailApp}
Let $d^{(i)}$ be the degree of agent $i.$ For some $\sigma\geq\sigma_k$ and for any $\zeta>1$ 
\begin{align*}
    \P\left(\Big|\widehat{\mu}_k^{(i)}(t)-{\mu}_k\Big|\geq \sigma\sqrt{\frac{2(\xi+1)\log t}{N_k^{(i)}(t)}}\right)\leq \frac{1}{\log \zeta}\frac{\log \left((d^{(i)}+1)t\right)}{t^{(\xi+1)\left(1-\frac{(\zeta-1)^2}{16}\right)}}
\end{align*}
\end{lemma}

\begin{proof}
\normalfont

Let $X_k$ be the sub-Gaussian random variable that models rewards drawn from arm $k.$ Then $X_k$ has mean $\mu_k$ and  variance proxy $\sigma_k.$ Then we have
\begin{align*}
\mathbb{E}\left(\exp(\lambda(X_k-\mu_k))\right)\leq \exp\left(\frac{\lambda^2 \sigma_k^2}{2}\right).
\end{align*}
Recall that $\indicate {A^{(i)}_\tau = k}$ is a $\mathcal{F}_{\tau-1}$ measurable random variable. Then we have
\begin{align*}
\expe\left(\exp\left(\lambda \left(X_k-\mu_k\right )\indicate {A^{(i)}_\tau = k}\indicate {(i,j) \in E }\right)\Big |\mathcal{F}_{\tau-1}\right)\leq  \exp\left(\frac{\lambda^2 \sigma_{k}^2}{2}\indicate {A^{(i)}_\tau = k}\indicate {(i,j) \in E_{\tau} }\right)
\end{align*}
Define a new random variable such that $\forall \tau>0.$
\begin{align*}
Y_k^{(i)}(\tau)&=\left(X_k-\mu_k\right )\sum_{j=1}^N\indicate {A^{(i)}_\tau = k}\indicate {(i,j) \in E_{\tau} }.
\end{align*}
Note that $\expe\left(Y_k^{(i)}(\tau)\right )=\expe\left(Y_k^{(i)}(\tau)|\mathcal{F}_{\tau-1}\right)=0.$ Let $Z_k^{(i)}(t)=\sum_{\tau=1}^{t}Y_k^{(i)}(\tau).$ For any $\lambda>0$
\begin{align*}
\expe\left(\exp(\lambda Y_k^{(i)}(\tau))|\mathcal{F}_{\tau-1}\right)
&=\mathbb{E}\left(\exp\left(\lambda \left(X_k-\mu_k\right )\sum_{j=1}^N\indicate {A^{(i)}_\tau = k}\indicate {(i,j) \in E_{\tau} }\right)\Big |\mathcal{F}_{\tau-1}\right)\\
&=\expe\left(\prod_{j=1}^K\exp\left(\lambda \left(X_k-\mu_k\right )\indicate {A^{(i)}_\tau = k}\indicate {(i,j) \in E_{\tau} }\right)\Big |\mathcal{F}_{\tau-1}\right)\\
&\overset{(a)}=\prod_{j=1}^N \expe\left(\exp\left(\lambda \left(X_k-\mu_k\right )\indicate {A^{(i)}_\tau = k}\indicate {(i,j) \in E_{\tau} }\right)\Big |\mathcal{F}_{\tau-1}\right)\\
&\leq \prod_{j=1}^N\exp\left(\frac{\lambda^2 \sigma_k^2}{2}\indicate {A^{(i)}_\tau = k}\indicate {(i,j) \in E_{\tau} }\right)\\
&=  \exp\left(\frac{\lambda^2 \sigma_k^2}{2}\sum_{j=1}^N\indicate {A^{(i)}_\tau = k}\indicate {(i,j) \in E_{\tau} }\right).
\end{align*}
Equality $(a)$ follows from the fact that random variables  $\left\{\exp\left(\lambda \left(X_k-\mu_k\right )\indicate {A^{(i)}_\tau = k}\indicate {(i,j) \in E_{\tau} }\right)\right \}_{j=1}^N$ are conditionally independent with respect to $\mathcal{F}_{\tau-1.}$
Since $\indicate {A^{(i)}_\tau = k},\indicate {(i,j) \in E_{\tau} }$ are $\mathcal{F}_{\tau-1}$ measurable random variable, and so
\begin{align*}
\expe\left(\exp\left(\lambda Y_k^{(i)}(\tau)-\frac{\lambda^2 \sigma_k^2}{2}\sum_{j=1}^N\indicate {A^{(i)}_\tau = k}\indicate {(i,j) \in E_{\tau} }\right)\Big |\mathcal{F}_{\tau-1}\right)\leq 1.
\end{align*}

Let $N_k^{(i)}(t)=\sum_{\tau=1}^t\sum_{j=1}^N\indicate {A^{(i)}_\tau = k}\indicate {(i,j) \in E_{\tau} }.$ Then we have 

Further, using the properties of conditional expectations 

\begin{align*}
& \expe\left(\exp\left(\lambda Z_k^{(i)}(t)-\frac{\lambda^2\sigma_k^2}{2}N_k^{(i)}(t)\right)\Big|\mathcal{F}_{t-1}\right)\leq \exp\left(\lambda Z_k^{(i)}(t-1)-\frac{\lambda^2\sigma_k^2}{2}N_k^{(i)}(t-1)\right).
\end{align*}
Thus we see that
\begin{align*}
\expe\left(\exp\left(\lambda Z_k^{(i)}(t)-\frac{\lambda^2 \sigma_k^2}{2}N_k^{(i)}(t)\right)\right) \leq 1.
\end{align*}

Note that we have
\begin{align*}
\P\left(\exp\left(\lambda Z_k^{(i)}(t)-\frac{\lambda^2 \sigma_{k}^2}{2}N_k^{(i)}(t)\right)\geq \exp\left(2\kappa \vartheta\right)\right)=\P\left(\lambda Z_k^{(i)}(t)-\frac{\lambda^2 \sigma_k^2}{2}N_k^{(i)}(t)\geq 2\kappa \vartheta\right)\\
=\P\left(\frac{Z_k^{(i)}(t)}{\sqrt{N_k^{(i)}(t)}}\geq  
\frac{2\kappa\vartheta}{\lambda}\sqrt{\frac{1}{N_k^{(i)}(t)}}+\frac{\sigma_k^{2}}{2}\lambda \sqrt{N_k^{(i)}(t)}
\right)
\end{align*}

Let $\zeta>1.$ Then  $
1\leq N_k^{(i)}(t)\leq \zeta^{D_{t}}$
where $D_{t}=\frac{\log ((d^{(i)}+1)t)}{\log \zeta}.$ For $\lambda_{l}=\frac{2}{\sigma_k}\sqrt{\frac{\kappa\vartheta}{\zeta^{l-1/2}}}$ and $\zeta^{l-1}\leq N_k^{(i)}(t)\leq \zeta^{l}$ we have
\begin{align*}
\frac{2\kappa\vartheta}{\lambda_l}\sqrt{\frac{1}{N_k^{(i)}(t)}}+\frac{\sigma_{k}^{2}}{2}\lambda_l \sqrt{N^k_{i}(t)}=\sigma_k\sqrt{\kappa\vartheta}\left(\sqrt{\frac{\zeta^{l-1/2}}{N_k^{(i)}(t)}}+\sqrt{\frac{N_k^{(i)}(t)}{\zeta^{l-1/2}}}\right)\leq \sqrt{\vartheta},
\end{align*}
where $\kappa=\frac{1}{\sigma_k^2\left(\zeta^{\frac{1}{4}}+\zeta^{-\frac{1}{4}}\right)^2}.$

Recall  from the Markov inequality that 
$\P(Y \geq a )\leq \frac{\expe(Y)}{a}$
for any positive random variable $Y$. Thus,
\begin{align*}
\P\left(\frac{Z_k^{(i)}(t)}{\sqrt{N_k^{(i)}(t)}}\geq\sqrt{ \vartheta}\right)\leq \sum_{l=1}^{D_{T}}\exp(-2\kappa\vartheta).
\end{align*}

Then we have,
\begin{align*}
\P\left(\frac{Z_k^{(i)}(t)}{N_k^{(i)}(t)}\geq\sqrt{ \frac{\vartheta}{N_k^{(i)}(t)}}\right)\leq \sum_{l=1}^{D_{T}}\exp(-2\kappa\vartheta)
\end{align*}
Substituting $\vartheta=2\sigma_k^{2}(\xi+1)\log t$ we get
\begin{align}
    \P\left(\Big|\widehat{\mu}_k^{(i)}(t)-{\mu}_k\Big |\geq \sigma_k\sqrt{\frac{2(\xi+1)\log t}{N_k^{(i)}(t)}}\right) \leq \frac{\log ((d^{(i)}+1)t)}{\log \zeta}\exp\left(-\frac{4(\xi+1)\log t}{\left(\zeta^{\frac{1}{4}}+\zeta^{-\frac{1}{4}}\right)^2}\right).\label{eq:addterm}
\end{align}
Since $\sigma\geq\sigma_k$ we have
\begin{align*}
    \P\left(\Big|\widehat{\mu}_k^{(i)}(t)-{\mu}_k\Big |\geq \sigma\sqrt{\frac{2(\xi+1)\log t}{N_k^{(i)}(t)}}\right) \leq \frac{\log ((d^{(i)}+1)t)}{\log \zeta}\exp\left(-\frac{4(\xi+1)\log t}{\left(\zeta^{\frac{1}{4}}+\zeta^{-\frac{1}{4}}\right)^2}\right).
\end{align*}
Note that $\forall \zeta>1$ we have
\begin{align}
  \frac{4}{\left(\zeta^{\frac{1}{4}}+\zeta^{-\frac{1}{4}}\right)^2  }\geq 1-\frac{(\zeta-1)^2}{16}
\end{align}
Then we have
\begin{align*}
    \P\left(\Big|\widehat{\mu}_k^{(i)}(t)-{\mu}_k\Big |\geq \sigma\sqrt{\frac{2(\xi+1)\log t}{N_k^{(i)}(t)}}\right) \leq \frac{1}{\log \zeta}\frac{\log ((d^{(i)}+1)t)}{t^{(\xi+1)\left(1-\frac{(\zeta-1)^2}{16}\right)}}.
\end{align*}

This concludes the proof of Lemma \ref{lem:tailApp}.
\end{proof}

\begin{lemma}\label{lem:tailsum}
Let $\zeta=1.3, \xi\geq 1.1,$ $d^{(i)}\geq 0$ and $t\in [T].$ Then we have 
\begin{align}
\sum_{t=1}^{T-1}\frac{1}{\log \zeta}\frac{\log \left((d^{(i)}+1) t\right)}{t^{(\xi+1)\left(1-\frac{(\zeta-1)^2}{16}\right)}}
\leq 
12\log (3(d^{(i)}+1)) +3\left(\log { (d^{(i)}+1)}+1\right)
\end{align}
\end{lemma}

\begin{proof}
\normalfont
For $\zeta=1.3$ we have $\frac{1}{\log \zeta}< 8.78.$ Further $(\xi+1)\left(1-\frac{(\zeta-1)^2}{16}\right)>2$ and $\forall t\geq 3$ we see that $\frac{\log \left((d^{(i)}+1) t\right)}{t^{(\xi+1)\left(1-\frac{(\zeta-1)^2}{16}\right)}}$ is monotonically decreasing. Thus we have
\begin{align}
\sum_{t=1}^{T-1}\frac{\log \left((d^{(i)}+1) t\right)}{t^{(\xi+1)\left(1-\frac{(\zeta-1)^2}{16}\right)}}\leq 1.362\log (3(d^{(i)}+1))+\int_3^{T-1}\frac{\log \left((d^{(i)}+1) t\right)}{t^{2}}dt\label{eq:intsum}
\end{align}
Let $z=(d^{(i)}+1) t.$ Then we have
\begin{align}
\int_3^{T-1}\frac{\log \left((d^{(i)}+1) t\right)}{t^{2}}dt
&
=(d^{(i)}+1)\int_{3(d^{(i)}+1)}^{(d^{(i)}+1) (T-1)}\frac{\log z}{z^{2}}dz
\\
&
=(d^{(i)}+1)\left[-\frac{\log z}{ z}-\frac{1}{z}\right]_{3((d^{(i)}+1)}^{(d^{(i)}+1) (T-1)}
\end{align}
Thus we have
\begin{align}
\int_3^{T-1}\frac{\log \left((d^{(i)}+1) t\right)}{t^{2}}dt
&\leq
(d^{(i)}+1)\left[\frac{\log (d^{(i)}+1)}{ 3(d^{(i)}+1)}+\frac{1}{3(d^{(i)}+1)}\right]
\\
&= \frac{1}{3}\log (d^{(i)}+1) +\frac{1}{3}
\label{eq:tailSum}
\end{align}
Recall that For $\zeta=1.3$ we have $\frac{1}{\log \zeta}< 8.78.$
Thus the proof of Lemma \ref{lem:tailsum} follows from  \eqref{eq:intsum} and \eqref{eq:tailSum}.
\end{proof}

Now we proceed to prove Theorem \ref{thm:regretB}. From definition of expected cumulative group regret and Lemmas \ref{lem:RegDecomp}, \ref{lem:tailApp} and \ref{lem:tailsum}  we have
\begin{align}
\expe \left[R(T)\right]
& \leq \sum_{k=2}^K\frac{8(\xi+1)\sigma}{\Delta_k}\Bar{\chi}(G)\log T + 4N\sum_{k=2}^K\Delta_k
\\
&+\sum_{i=1}^N\left(12\log (3(d^{(i)}+1)) +3\log { (d^{(i)}+1)}\right)\sum_{k=2}^K\Delta_k
\label{eq:regretf}
\end{align}

This concludes the proof of Theorem \ref{thm:regretB}.




\section{Proof of Theorem \ref{thm:comcostB}}\label{sec:comcostB}
Recall that all the agents communicate their rewards and arm ids at time $t=1$. Then the expected communication cost can be given as
\begin{align}
\expe\left[L(T)\right]=\sum_{i=1}^N\sum_{t=1}^{T-1} \P\left(A_{t}^{(i)}\neq \argmax_{k\in [K]}\widehat{\mu}_{k}^{(i)}(t-1)\right).\label{eq:comcost}
\end{align}
Note that we have
\begin{align}
    \sum_{i=1}^N\sum_{t=1}^{T-1}\P\left(A_{t}^{(i)}\neq \argmax_{k\in [K]}\widehat{\mu}_{k}^{(i)}(t-1)\right) &=  \sum_{i=1}^N\sum_{t=1}^{T-1}\P\left(A_{t}^{(i)}=1,A_{t}^{(i)}\neq \argmax_{k\in [K]}\widehat{\mu}_{k}^{(i)}(t-1)\right)
    \nonumber \\
      &+ \sum_{i=1}^N\sum_{t=1}^{T-1}\P\left(A_{t}^{(i)}\neq 1,A_{t}^{(i)}\neq \argmax_{k\in [K]}\widehat{\mu}_{k}^{(i)}(t-1)\right).\label{eq:comdec}
\end{align}
For all agents we first upper bound the expected number of times they shares rewards and actions with their neighbors until time $T$ when they pull a suboptimal arm: 
\begin{align}
\sum_{i=1}^N\sum_{t=1}^{T-1}\P\left(A_{t}^{(i)}\neq 1,A_{t}^{(i)}\neq \argmax_{k\in [K]}\widehat{\mu}_{k}^{(i)}(t-1)\right)\leq \sum_{i=1}^N\sum_{t=1}^{T-1}\P\left(A_{t}^{(i)}\neq 1\right)\leq \sum_{i=1}^N\sum_{k=2}^K\expe \left[n_{k}^{(i)}(T)\right]. \label{eq:subOpt}
\end{align}
Next for all agents we upper bound the expected number of times they shares rewards and actions with their neighbors until time $T$ when they pull the optimal arm as follows. Let $k_t^*$ be the suboptimal arm with highest estimated expected reward for agents $i$ at time $t.$ This can be stated as $k_t^*=\argmax_{k\neq 1,k\in [K]} \widehat{\mu}_k^{(i)}(t)$. Note that $\forall i,t$ we have
\begin{align*}
\left\{A_{t+1}^{(i)}=1,A_{t+1}^{(i)}\neq \argmax_{k\in [K]}\widehat{\mu}_{k}^{(i)}(t)\right\}
\subseteq \left \{\widehat{\mu}_{1}^{(i)}(t)\leq \mu_{1}-C_{1}^{(i)}(t)\right\}\\
\cup\left \{A_{t+1}^{(i)}=1,\widehat{\mu}_{1}^{(i)}(t)\geq  \mu_{1}-C_{1}^{(1)}(t), \widehat{\mu}_{k_{t}^*}^{(i)}(t)\geq \widehat{\mu}_{1}^{(i)}(t)\right\}.
\end{align*}
Thus, we have
\begin{align}
\sum_{t=1}^{T-1}\P\left(A_{t}^{(i)}=1,A_{t}^{(i)}\neq \argmax_{k\in [K]}\widehat{\mu}_{k}^{(i)}(t-1)\right)
\leq \sum_{t=1}^T\P \left(\widehat{\mu}_{1}^{(i)}(t-1)\leq \mu_{1}-C_{1}^{(i)}(t-1)\right)\nonumber\\
+\sum_{t=1}^T\P\left (A_{t+1}^{(i)}=1,\widehat{\mu}_{1}^{(i)}(t-1)\geq  \mu_{1}-C_{1}^{(1)}(t-1), \widehat{\mu}_{k_{t}^*}^{(i)}(t-1)\geq \widehat{\mu}_{1}^{(i)}(t-1)\right).\label{eq:optdec}
\end{align}
Note that the first term on the right hand side of the above equation is the summation tail probabilities of the estimate of the optimal arm. Now we proceed to upper bound the second term as follows. Let $\tau_{1}^{(i)}$ denote the maximum time step when the total number of times agent $i$ pulled the optimal arm and the total number of observations it received from its neighbors about the optimal arm is at most $\Bar{\eta}$. This can be stated as $\tau_{1}^{(i)} := \max \{t\in [T] : N_1^{(i)}(t)\leq \Bar{\eta}\}$. Recall that $N_1^{(i)}(t)\geq n_1^{(i)}(t).$ Thus we have that $n_1^{(i)}\left (t\right)\leq \Bar{\eta}, \forall t\leq\tau_1^{(i)}.$

Note that we have
\begin{align}
&\sum_{t=1}^{T-1}\P\left (A_{t}^{(i)}=1,\widehat{\mu}_{1}^{(i)}(t-1)\geq  \mu_{1}-C_{1}^{(1)}(t-1), \widehat{\mu}_{k_{t}^*}^{(i)}(t-1)\geq \widehat{\mu}_{1}^{(i)}(t-1)\right)
\nonumber \\
&\leq\sum_{t=1}^{\tau_1^{(i)}}\P\left (A_{t}^{(i)}=1,\widehat{\mu}_{1}^{(i)}(t-1)\geq  \mu_{1}-C_{1}^{(1)}(t-1), \widehat{\mu}_{k_{t}^*}^{(i)}(t-1)\geq \widehat{\mu}_{1}^{(i)}(t-1)\right)
\nonumber \\
&+\sum_{t>\tau_1^{(i)}}^{T-1}\P\left (A_{t}^{(i)}=1,\widehat{\mu}_{1}^{(i)}(t-1)\geq  \mu_{1}-C_{1}^{(1)}(t-1), \widehat{\mu}_{k_{t-1}^*}^{(i)}(t-1)\geq \widehat{\mu}_{1}^{(i)}(t-1)\right)
\nonumber \\
&\leq \Bar{\eta}+1+\sum_{t>\tau_1^{(i)}}^{T-2}\P\left (A_{t+1}^{(i)}=1,\widehat{\mu}_{1}^{(i)}(t)\geq  \mu_{1}-C_{1}^{(1)}(t), \widehat{\mu}_{k_{t}^*}^{(i)}(t)\geq \widehat{\mu}_{1}^{(i)}(t),N_1^{(i)}(t)> \Bar{\eta}\right).\label{eq:optitail}
\end{align}

If agent $i$ pulls the optimal arm at time $t$ we have $Q_{1}^{(i)}(t-1)\geq Q_{k_{t-1}^*}^{(i)}(t-1).$ Further, if $\widehat{\mu}_{k_{t-1}^*}^{(i)}(t-1)\geq\widehat{\mu}_{1}^{(i)}(t-1)$ then we have $C_{k_{t-1}^*}^{(i)}(t-1)<C_{1}^{(i)}(t-1).$ Let $\Bar{\eta}=\frac{8\sigma(\xi+1)}{\Bar{\Delta}^2}\log T.$ Then we have
\begin{align}
\sum_{t>\tau_1^{(i)}}^{T-2}& \P\left (A_{t+1}^{(i)}=1,\widehat{\mu}_{1}^{(i)}(t)\geq  \mu_{1}-C_{1}^{(1)}(t), \widehat{\mu}_{k_{t}^*}^{(i)}(t)\geq \widehat{\mu}_{1}^{(i)}(t),N_1^{(i)}(t)> \Bar{\eta}\right)
\nonumber\\
& \leq \sum_{t>\tau_1^{(i)}}^{T-2}\P\left (A_{t+1}^{(i)}=1,\widehat{\mu}_{1}^{(i)}(t)\geq  \mu_{1}-C_{1}^{(1)}(t), \widehat{\mu}_{k_{t-1}^*}^{(i)}(t)\geq \widehat{\mu}_{1}^{(i)}(t),\mu_1>\mu_{k^*_{t}}+2C_1^{(i)}(t)\right)
\nonumber\\
& \leq\sum_{t>\tau_1^{(i)}}^{T-2}\P\left(\widehat{\mu}_{k_{t-1}^*}^{(i)}(t)\geq \mu_{1}-C_{1}^{(i)}(t),\mu_{1}>\mu_{k_{t}^*}+2C_{1}^{(i)}(t)\right)
\nonumber\\
& \leq\sum_{t>\tau_1^{(i)}}^{T-2} 
\P\left(\widehat{\mu}_{k_{t}^*}^{(i)}(t)\geq \mu_{k_{t}^*}+C_{k_{t}^*}^{(i)}(t)\right).\label{eq:opti2}
\end{align}
From (\ref{eq:optdec}), (\ref{eq:optitail}) and (\ref{eq:opti2}) we have
\begin{align}
 \sum_{t=1}^T\P\left(A_{t}^{(i)}=1,A_{t}^{(i)}\neq \argmax_{k\in [K]}\widehat{\mu}_{k}^{(i)}(t-1)\right) \leq\frac{8\sigma(\xi+1)}{\Bar{\Delta}^2}\log T\\
 +\sum_{t=1}^{T-1}\P \left(\widehat{\mu}_{1}^{(i)}(t-1)\leq \mu_{1}-C_{1}^{(i)}(t-1)\right)
 +\sum_{t=1}^{T-1} 
\P\left(\widehat{\mu}_{k_{t-1}^*}^{(i)}(t-1)\geq \mu_{k_{t-1}^*}+C_{k_{t-1}^*}^{(i)}(t-1)\right)  \label{eq:tailsum}
\end{align}
From (\ref{eq:tailsum}) and Lemma \ref{lem:tailApp} we have
\begin{align}
 \sum_{t=1}^T\P\left(A_{t}^{(i)}=1,A_{t}^{(i)}\neq \argmax_{k\in [K]}\widehat{\mu}_{k}^{(i)}(t-1)\right)& \leq   \frac{8\sigma(\xi+1)}{\Bar{\Delta}^2}\log T+2\sum_{t=1}^T\frac{1}{\log \zeta}\frac{\log \left((d^{(i)}+1)t\right)}{t^{(\xi+1)\left(1-\frac{(\zeta-1)^2}{16}\right)}}\label{eq:pullopt}
\end{align}
The proof of Theorem \ref{thm:comcostB} follows from (\ref{eq:comcost}), (\ref{eq:comdec}), (\ref{eq:subOpt}), (\ref{eq:pullopt}) and Theorem \ref{thm:regretB}.

\section{Proof of Theorem \ref{thm:regretMP}}\label{sec:regretMP}
In section we follow an approach similar to Section \ref{sec:regretB}. Recall that $G_{\gamma}$ is the $\gamma^{\mathrm{th}}$ power graph of $G.$ Thus each pair of vertices in $G_{\gamma}$ are adjacent if and only if they distance between them in $G$ is at most $\gamma.$  We begin the proof of Theorem \ref{thm:regretMP} by proving a lemma similar to Lemma \ref{lem:RegDecomp}.

\begin{lemma}\label{lem:RegDecompMP}
Let $\Bar{\chi}(G_{\gamma})$ is the clique covering number of graph $G_{\gamma}.$ Let $\eta_k=\left(\frac{8(\xi+1)\sigma^2}{\Delta^2_k}\right)\log T.$ Then we have
\begin{align*}
    \sum_{i=1}^N\expe[n^{(i)}_k(T)]  \leq \Bar{\chi}(G_{\gamma}) \eta_k+N+(N-\Bar{\chi}(G_{\gamma}))(\gamma-1)\\
    +\sum_{i=1}^N\sum_{t=1}^{T-1}\left[\P \left(\widehat{\mu}_{1}^{(i)}(t)\leq \mu_{1}-C_{1}^{(i)}(t)\right)+\P \left(\widehat{\mu}_{k}^{(i)}(t)\geq \mu_{k}+C_{k}^{(i)}(t)\right)\right]
\end{align*}
\end{lemma}
\begin{proof}
Let $\EuScript{C}_{\gamma}$ be a non overlapping clique covering of $G_{\gamma}$. Note that for each suboptimal arm $k > 1$ we have
\begin{align}
    \sum_{i=1}^N\expe[n^{(i)}_k(T)]
    & = \sum_{\mathcal{C}\in \EuScript{C}_{\gamma}}\sum_{i\in \mathcal{C}}\sum_{t=1}^T \P\left(A^{(i)}_t = k\right) \label{eq:rearrangeMP}
\end{align}

 Let ${\tau}_{k,\mathcal{C}}$ be the maximum time step such that the total number of messages about pulls from arm $k$ initiated by agents in the clique $\mathcal{C}$ is at most $\eta_k+\left(|\mathcal{C}|-1\right)(\gamma-1).$ This can be stated as \begin{align*}
{\tau}_{k,\mathcal{C}} := \max \left\{t\in [T] :  \sum_{i\in \mathcal{C}}\sum_{\tau=1}^t\indicate{A_{\tau}^{(i)}=k,A_{\tau}^{(i)}\neq \argmax_{l\in [K]}\widehat{\mu}_{l}^{(i)}(\tau-1)}\leq \eta_k+\left(|\mathcal{C}|-1\right)(\gamma-1)\right\}.
\end{align*} 
Further for all $i\in \mathcal{C}$ we have $
N_k^{(i)}(t)> \eta_k, \forall t> {\tau}_{k,\mathcal{C}}.$ 
We analyse the expected number of times all agents pull suboptimal arm $k$ as follows.
\begin{align}
\sum_{\mathcal{C}\in \EuScript{C}}\sum_{i\in \mathcal{C}}\sum_{t=1}^T \indicate{A^{(i)}_t = k}  &= \sum_{\mathcal{C}\in \EuScript{C}}\sum_{i\in \mathcal{C}}\sum_{t=1}^{{\tau}_{k,\mathcal{C}}} \indicate{A^{(i)}_t = k}+  \sum_{\mathcal{C}\in \EuScript{C}}\sum_{i\in \mathcal{C}}\sum_{t>{\tau}_{k,\mathcal{C}}}^T \indicate{A^{(i)}_t = k, N_k^{(i)}(t-1) >\eta_k} \label{eq:regdecMP}
\end{align}
Taking the expectation of (\ref{eq:regdecMP}) we have
\begin{align}
\sum_{\mathcal{C}\in \EuScript{C}}\sum_{i\in \mathcal{C}}\sum_{t=1}^T \P\left(A^{(i)}_t = k\right)  =   \sum_{\mathcal{C}\in \EuScript{C}}\sum_{i\in \mathcal{C}}\sum_{t=1}^{\tau_{k,\mathcal{C}}} \P\left(A^{(i)}_t = k\right)+  \sum_{\mathcal{C}\in \EuScript{C}}\sum_{i\in \mathcal{C}}\sum_{t>\tau_{k,\mathcal{C}}}^T \P\left(A^{(i)}_t = k,N_k^{(i)}(t-1) >\eta_k\right)\label{eq:regDecComMP}
\end{align}
Now we proceed to upper bound the first term of right hand side of (\ref{eq:regdecMP}) as follows. Note that we have
\begin{align}
\sum_{i\in \mathcal{C}}\sum_{t=1}^{\tau_{k,\mathcal{C}}} \indicate{A^{(i)}_t = k} & =  \sum_{i\in \mathcal{C}}\sum_{t=1}^{\tau_{k,\mathcal{C}}}\indicate{A_{t}^{(i)}=k,A_{t}^{(i)}\neq \argmax_{l\in [K]}\widehat{\mu}_{l}^{(i)}(t-1)}
\nonumber \\
& +\sum_{i\in \mathcal{C}}\sum_{t=1}^{\tau_{k,\mathcal{C}}}\indicate{A_{t}^{(i)}=k,A_{t}^{(i)}= \argmax_{l\in [K]}\widehat{\mu}_{l}^{(i)}(t-1)}
\nonumber \\
&\leq \eta_k +\left(|\mathcal{C}|-1\right)(\gamma-1)+ \sum_{i\in \mathcal{C}}\sum_{t=1}^{\tau_{k,\mathcal{C}}}\indicate{A_{t}^{(i)}=k,A_{t}^{(i)}= \argmax_{l\in [K]}\widehat{\mu}_{l}^{(i)}(t-1)}\label{eq:noMP}
\end{align}

Taking the expectation of (\ref{eq:noMP}) we have
\begin{align}
\sum_{i\in \mathcal{C}}\sum_{t=1}^{\tau_{k,\mathcal{C}}} \P\left(A^{(i)}_t = k\right)\leq \eta_k+\left(|\mathcal{C}|-1\right)(\gamma-1)+ \sum_{i\in \mathcal{C}}\sum_{t=1}^{\tau_{k,\mathcal{C}}}\P\left(A_{t}^{(i)}=k,A_{t}^{(i)}= \argmax_{l\in [K]}\widehat{\mu}_{l}^{(i)}(t-1)\right)\label{eq:suboptMP}
\end{align}

Now we proceed to upper bound last term of (\ref{eq:suboptMP}) as follows. Note that for any suboptimal arm $k$ we have,
\begin{align}
\P&\left(A_{t+1}^{(i)}=k,A_{t+1}^{(i)}= \argmax_{l\in [K]}\widehat{\mu}_{l}^{(i)}(t)\right) 
\nonumber\\\ 
&  \leq\P\left(\widehat{\mu}_k^{(i)}(t)+C_k^{(i)}(t)\geq \widehat{\mu}_1^{(i)}(t)+C_1^{(i)}(t), \widehat{\mu}_k^{(i)}(t)\geq \widehat{\mu}_1^{(i)}(t), \:\:
\widehat{\mu}_1^{(i)}(t)\leq \mu_{1}-C_1^{(i)}(t)\right)
\nonumber\\
& +\P\left(\widehat{\mu}_k^{(i)}(t)+C_k^{(i)}(t)\geq \widehat{\mu}_1^{(i)}(t)+C_1^{(i)}(t), \widehat{\mu}_k^{(i)}(t)\geq \widehat{\mu}_1^{(i)}(t), \:\:
\widehat{\mu}_1^{(i)}(t)> \mu_{1}-C_1^{(i)}(t)\right)
\nonumber\\\ 
&  \leq \P\left(\widehat{\mu}_1^{(i)}(t)\leq \mu_{1}-C_1^{(i)}(t)\right)\nonumber\\\ 
&  
+\P\left(\widehat{\mu}_k^{(i)}(t)+C_k^{(i)}(t)\geq \widehat{\mu}_1^{(i)}(t)+C_1^{(i)}(t), \widehat{\mu}_k^{(i)}(t)\geq \widehat{\mu}_1^{(i)}(t), \:\:
\widehat{\mu}_1^{(i)}(t)> \mu_{1}-C_1^{(i)}(t)\right)\label{eq:condeqMP}
\end{align}
Now we proceed to upper bound the last term of (\ref{eq:condeqMP}) as follows. Note that we have
\begin{align}
\P& \left(\widehat{\mu}_k^{(i)}(t)+C_k^{(i)}(t)\geq \widehat{\mu}_1^{(i)}(t)+C_1^{(i)}(t), \widehat{\mu}_k^{(i)}(t)\geq \widehat{\mu}_1^{(i)}(t), \:\:
\widehat{\mu}_1^{(i)}(t)> \mu_{1}-C_1^{(i)}(t)\right)
\nonumber \\
& \leq\P\left(\widehat{\mu}_k^{(i)}(t)+C_k^{(i)}(t)>\mu_{1}-C_1^{(i)}(t)+C_k^{(i)}(t)\right)
\nonumber \\
& \leq \P\left(\widehat{\mu}_k^{(i)}(t)\geq \mu_k+C_k^{(i)}(t)\right).\label{eq:optsamMP}
\end{align}

From (\ref{eq:condeqMP}) and (\ref{eq:optsamMP}) we have
\begin{align}
\P\left(A_{t}^{(i)}=k,A_{t}^{(i)}= \argmax_{l\in [K]}\widehat{\mu}_{l}^{(i)}(t-1)\right)  \leq \P\left(\widehat{\mu}_1^{(i)}(t-1)\leq \mu_{1}-C_1^{(i)}(t-1)\right)\\
+\P\left(\widehat{\mu}_k^{(i)}(t-1)\geq \mu_k+C_k^{(i)}(t-1)\right). \label{eq:addtailMP}
\end{align}

From (\ref{eq:regDecComMP}), (\ref{eq:suboptMP}) and (\ref{eq:addtailMP}) we have
\begin{align}
\sum_{\mathcal{C}\in \EuScript{C}}\sum_{i\in \mathcal{C}}\sum_{t=1}^T \P\left(A^{(i)}_t = k\right) \nonumber\\
\leq \sum_{\mathcal{C}\in \EuScript{C}}\eta_k+ \sum_{\mathcal{C}\in \EuScript{C}}\sum_{i\in \mathcal{C}}\sum_{t=1}^{\tau_{k,\mathcal{C}}}\left[\P \left(\widehat{\mu}_{1}^{(i)}(t)\leq \mu_{1}-C_{1}^{(i)}(t)\right)+\P \left(\widehat{\mu}_{k}^{(i)}(t)\geq \mu_{k}+C_{k}^{(i)}(t)\right)\right]
\nonumber\\
+ \sum_{\mathcal{C}\in \EuScript{C}}\left(|\mathcal{C}|-1\right)(\gamma-1)+ \sum_{\mathcal{C}\in \EuScript{C}}\sum_{i\in \mathcal{C}}\sum_{t>\tau_{k,\mathcal{C}}}^T \P\left(A^{(i)}_t = k,N_k^{(i)}(t-1) >\eta_k\right)
\nonumber \\
\leq \Bar{\chi}(G_{\gamma})\eta_k+ \sum_{i=1}^N\sum_{t=1}^{\tau_{k,\mathcal{C}}}\left[\P \left(\widehat{\mu}_{1}^{(i)}(t)\leq \mu_{1}-C_{1}^{(i)}(t)\right)+\P \left(\widehat{\mu}_{k}^{(i)}(t)\geq \mu_{k}+C_{k}^{(i)}(t)\right)\right]
\nonumber\\
+N+ \left(N-\Bar{\chi}(G_{\gamma})\right)(\gamma-1)+\sum_{\mathcal{C}\in \EuScript{C}}\sum_{i\in \mathcal{C}}\sum_{t>\tau_{k,\mathcal{C}}}^{T-1} \P\left(A^{(i)}_{t+1} = k,N_k^{(i)}(t) >\eta_k\right).
\label{eq:subsamplMP}
\end{align}

The proof of Lemma \ref{lem:RegDecompMP} follows from (\ref{eq:rearrangeMP}), (\ref{eq:subsamplMP}) and Lemma \ref{lem:tailrestate}. 
\end{proof}

Now we proceed to prove Theorem \ref{thm:regretMP} as follows. We start by obtaining a modified tail bound similar to the result in Lemma \ref{lem:tailApp}. Note that $\forall i,k,t$ we have $1\leq N_k^{(i)}(t)< d_{\gamma}^{(i)}t.$ Thus considering $D_t=\frac{\log \left(\left(d_{\gamma}^{(i)}+1\right)t\right)}{\log \zeta}$ for any $\zeta>1$ in Lemma \ref{lem:tailApp} we get
\begin{align}
  \P\left(\Big|\widehat{\mu}_k^{(i)}(t)-{\mu}_k\Big|\geq \sigma\sqrt{\frac{2(\xi+1)\log t}{N_k^{(i)}(t)}}\right)\leq \frac{1}{\log \zeta}\frac{\log \left((d_{\gamma}^{(i)}+1)t\right)}{t^{(\xi+1)\left(1-\frac{(\zeta-1)^2}{16}\right)}}.\label{eq:modtailMP}
\end{align}
The proof of Theorem \ref{thm:regretMP} follows from Lemmas \ref{lem:tailsum}, \ref{lem:RegDecompMP} and (\ref{eq:modtailMP}).


\section{Proof of Theorem \ref{thm:comcostMP}}\label{sec:comcostMP}
Following a similar approach to the proof of Theorem \ref{thm:comcostB} we obtain
\begin{align}
\sum_{i=1}^N\sum_{t=1}^{T-1}\P\left(A_{t}^{(i)}\neq 1,A_{t}^{(i)}\neq \argmax_{k\in [K]}\widehat{\mu}_{k}^{(i)}(t-1)\right)\leq \sum_{i=1}^N\sum_{t=1}^{T-1}\P\left(A_{t}^{(i)}\neq 1\right)\leq \sum_{i=1}^N\sum_{k=2}^K\expe \left[n_{k}^{(i)}(T)\right]. \label{eq:subOptMP}
\end{align}
Similarly we get 
\begin{align}
 \sum_{t=1}^{T-1}\P\left(A_{t}^{(i)}=1,A_{t}^{(i)}\neq \argmax_{k\in [K]}\widehat{\mu}_{k}^{(i)}(t-1)\right)& \leq   \frac{8\sigma(\xi+1)}{\Bar{\Delta}^2}\log T+2\sum_{t=1}^T\frac{1}{\log \zeta}\frac{\log \left((d_{\gamma}^{(i)}+1)t\right)}{t^{(\xi+1)\left(1-\frac{(\zeta-1)^2}{16}\right)}}\label{eq:pulloptMP}
\end{align}
From (\ref{eq:subOptMP}) and (\ref{eq:pulloptMP}) we have
\begin{align}
\sum_{i=1}^N\sum_{t=1}^T\P\left(A_{t}^{(i)}\neq \argmax_{k\in [K]}\widehat{\mu}_{k}^{(i)}(t-1)\right)\leq \sum_{i=1}^N\sum_{k=2}^K\expe \left[n_{k}^{(i)}(T)\right]
\nonumber\\
+\sum_{i=1}^N\frac{8\sigma(\xi+1)}{\Bar{\Delta}^2}\log T+2\sum_{i=1}^N\sum_{t=1}^T\frac{1}{\log \zeta}\frac{\log \left((d_{\gamma}^{(i)}+1)t\right)}{t^{(\xi+1)\left(1-\frac{(\zeta-1)^2}{16}\right)}}\label{eq:initiated}
\end{align}
Note that (\ref{eq:initiated}) is the expected number of messages initiated by all the agents. Recall that in ComEx-MPUCB a message initiated by agent $i$ is subsequently passed by agents within a $\gamma-1$ distance in graph $G.$ Thus we have
\begin{align}
\expe \left[L(T)\right]\leq \sum_{i=1}^N(d_{\gamma-1}^{(i)}+1)\sum_{t=1}^T\P\left(A_{t}^{(i)}\neq \argmax_{k\in [K]}\widehat{\mu}_{k}^{(i)}(t-1)\right)\label{eq:repeatmessage}
\end{align}
From (\ref{eq:initiated}) and (\ref{eq:repeatmessage}) we have
\begin{align}
\expe \left[L(T)\right]\leq \sum_{i=1}^N(d_{\gamma-1}^{(i)}+1)\sum_{k=2}^K\expe \left[n_{k}^{(i)}(T)\right]
+\sum_{i=1}^N(d_{\gamma-1}^{(i)}+1)\frac{8\sigma(\xi+1)}{\Bar{\Delta}^2}\log T
\nonumber\\
+2\sum_{i=1}^N(d_{\gamma-1}^{(i)}+1)\sum_{t=1}^T\frac{1}{\log \zeta}\frac{\log \left((d_{\gamma}^{(i)}+1)t\right)}{t^{(\xi+1)\left(1-\frac{(\zeta-1)^2}{16}\right)}}\label{eq:comregMP}
\end{align}
From (\ref{eq:modtailMP}), (\ref{eq:comregMP}) and Lemma \ref{lem:RegDecompMP} we have
\begin{align}
\expe \left[L(T)\right]\leq \sum_{i=1}^N(d_{\gamma-1}^{(i)}+1)\sum_{k=2}^K\left(\Bar{\chi}(G_{\gamma}) \eta_k+N+(N-\Bar{\chi}(G_{\gamma}))(\gamma-1)\right)
\nonumber
\\
+\sum_{i=1}^N(d_{\gamma-1}^{(i)}+1)\frac{8\sigma(\xi+1)}{\Bar{\Delta}^2}\log T
+2K\sum_{i=1}^N(d_{\gamma-1}^{(i)}+1)\sum_{t=1}^T\frac{1}{\log \zeta}\frac{\log \left((d_{\gamma}^{(i)}+1)t\right)}{t^{(\xi+1)\left(1-\frac{(\zeta-1)^2}{16}\right)}}\label{eq:regfinal}
\end{align}
Recall that $\eta_k=\frac{8\sigma(\xi+1)}{\Delta_k^2}\log T.$ Thus the proof of Theorem \ref{thm:comcostMP} follows from (\ref{eq:regfinal}) and Lemma \ref{lem:tailsum}.

\section{Proof of Theorem \ref{thm:regretLFUCB}} \label{sec:regretLFUCB}
We follow a similar approach to proof of Theorem \ref{thm:regretMP}. We begin the proof by providing a lemma similar to Lemma \ref{lem:RegDecompMP}.

\begin{lemma}\label{lem:RegDecompLF}
Let $\Bar{\gamma}(G_{\gamma})$ is the dominating number of graph $G_{\gamma}.$ Let $\eta_k=\left(\frac{8(\xi+1)\sigma^2}{\Delta^2_k}\right)\log T.$ Then we have
\begin{align*}
    \sum_{i=1}^N\expe[n^{(i)}_k(T)]   \leq  \Bar{\gamma}(G_{\gamma})\eta_k+N+(N-\Bar{\gamma}(G_{\gamma}))(3\gamma-1)
    \\
     +\sum_{i\in V^{\prime}_{\gamma}}\left(\Big |\mathcal{N}_{\gamma}^{(i)}\Big |+1\right)\sum_{t=1}^{T-1}\left[\P \left(\widehat{\mu}_{1}^{(i)}(t)\leq \mu_{1}-C_{1}^{(i)}(t)\right)+\P \left(\widehat{\mu}_{k}^{(i)}(t)\geq \mu_{k}+C_{k}^{(i)}(t)\right)\right]
\end{align*}
where $V^{\prime}_{\gamma}$ is the maximal dominating set of $G_{\gamma}$ and $\mathcal{N}_{\gamma}^{(i)}$ is the set of followers of leader $i.$
\end{lemma}
\begin{proof}
Recall that $V^{\prime}_{\gamma}$ is the maximal dominating set of $G_{\gamma}.$ Let $\mathcal{N}_{\gamma}^{(i)}$ be the set of followers of leader $i.$ Then for each suboptimal arm $k > 1$ we have
\begin{align}
    \sum_{i=1}^N\expe[n^{(i)}_k(T)]
    & = \sum_{i\in V^{\prime}_{\gamma}}\left(\sum_{t=1}^T \P\left(A^{(i)}_t = k\right)+\sum_{j\in \mathcal{N}_{\gamma}^{(i)} }\sum_{t=1}^T \P\left(A^{(j)}_t = k\right)\right) \label{eq:rearrangeLF}
\end{align}

 Let ${\tau}^{(i)}_{k}$ be the maximum time step such that the total number of times agent $i$ pulls arm $k$ and the number of times agents in $\mathcal{N}_{\gamma}^{(i)}$ initiated messages about pulls from arm $k$ is at most $\eta_k+\mathcal{N}_{\gamma}^{(i)}(\gamma-1).$ This can be stated as \begin{align*}
\tau^{(i)}_{k} := \max \left\{t\in [T] :  \sum_{\tau=1}^{t}\indicate{A_{\tau}^{(i)}=k}+\sum_{j\in \mathcal{N}_{\gamma}^{(i)}}\sum_{\tau=1}^t\indicate{A_{\tau}^{(i)}=k,A_{\tau}^{(i)}\neq \argmax_{l\in [K]}\widehat{\mu}_{l}^{(i)}(\tau-1)}\right.
\\
\left.
\leq \eta_k+\mathcal{N}_{\gamma}^{(i)}(\gamma-1)\right\}.
\end{align*} 
Then we have $
N_k^{(i)}(t)> \eta_k, \forall t> \tau^{(i)}_{k}.$ 
We analyse the expected number of times all agents pull suboptimal arm $k$ as follows. Let $d(i,j)$ be the distance between agents $i$ and $j$ in graph $G.$ Then note that for any $j\in \mathcal{N}_{\gamma}^{(i)}$ we have $A_t^{(j)}=A_{t-d(i,j)}^{(i)}$ and $d(i,j)\leq \gamma.$ 
\begin{align}
\sum_{i\in V^{\prime}_{\gamma}}\left\{\sum_{t=1}^T \indicate{A^{(i)}_t = k}+\sum_{j\in \mathcal{N}_{\gamma}^{(i)} }\sum_{t=1}^T \indicate{A^{(j)}_t = k}\right\}  \leq \sum_{i\in V^{\prime}_{\gamma}}\left\{\sum_{t=1}^{\tau^{(i)}_{k}} \indicate{A^{(i)}_t = k}
\right.
\nonumber\\
\left.
+\sum_{j\in \mathcal{N}_{\gamma}^{(i)} }\sum_{t=d(i,j)}^{\tau^{(i)}_{k}} \indicate{A^{(j)}_t = k}\right\}
\nonumber\\
+\sum_{i\in V^{\prime}_{\gamma}}\left\{\sum_{t>\tau^{(i)}_{k}}^T \indicate{A^{(i)}_t = k}+\sum_{j\in \mathcal{N}_{\gamma}^{(i)} }\sum_{t>\tau^{(i)}_{k}}^{T-d(i,j)} \indicate{A^{(i)}_t = k}\right\} +\sum_{i\in V^{\prime}_{\gamma}}\sum_{j\in \mathcal{N}_{\gamma}^{(i)}}2d(i,j)
\nonumber\\
\leq\sum_{i\in V^{\prime}_{\gamma}}\left\{\sum_{t=1}^{\tau^{(i)}_{k}} \indicate{A^{(i)}_t = k}+\sum_{j\in \mathcal{N}_{\gamma}^{(i)} }\sum_{t=d(i,j)}^{\tau^{(i)}_{k}} \indicate{A^{(j)}_t = k}\right\}+\sum_{i\in V^{\prime}_{\gamma}}\left(\Big |\mathcal{N}_{\gamma}^{(i)}\Big |+1\right)\sum_{t>\tau^{(i)}_{k}}^T \indicate{A^{(i)}_t = k}
\nonumber\\
+2(N -\Bar{\gamma}(G_{\gamma}))\gamma
\label{eq:regdecLF}
\end{align}
Now we proceed to upper bound the first two terms of right hand side of (\ref{eq:regdecLF}) as follows. Note that we have
\begin{align}
\sum_{i\in V^{\prime}_{\gamma}}\left\{\sum_{t=1}^{\tau^{(i)}_{k}} \indicate{A^{(i)}_t = k}+\sum_{j\in \mathcal{N}_{\gamma}^{(i)} }\sum_{t=d(i,j)}^{\tau^{(i)}_{k}} \indicate{A^{(j)}_t = k}\right\} 
\nonumber\\
=\sum_{i\in V^{\prime}_{\gamma}}\left\{\sum_{t=1}^{\tau^{(i)}_{k}} \indicate{A^{(i)}_t = k}+\sum_{j\in \mathcal{N}_{\gamma}^{(i)} }\sum_{t=d(i,j)}^{\tau^{(i)}_{k}} \indicate{A^{(j)}_t = k,A_{t-d(i,j)}^{(i)}\neq \argmax_{l\in [K]}\widehat{\mu}_{l}^{(i)}(t-d(i,j)-1)} \right.
\nonumber\\
\left.+  \sum_{j\in \mathcal{N}_{\gamma}^{(i)} }\sum_{t=d(i,j)}^{\tau^{(i)}_{k}} \indicate{A^{(j)}_t = k,A_{t-d(i,j)}^{(i)}= \argmax_{l\in [K]}\widehat{\mu}_{l}^{(i)}(t-d(i,j)-1)}\right\}
\nonumber\\
\leq \sum_{i\in V^{\prime}_{\gamma}}\left(\eta_k+\mathcal{N}_{\gamma}^{(i)}(\gamma-1)\right)+ \sum_{i\in V^{\prime}_{\gamma}}\Big |\mathcal{N}_{\gamma}^{(i)}\Big |\sum_{t=1}^{\tau_{k}^{(i)}}\indicate{A_{t}^{(i)}=k,A_{t}^{(i)}= \argmax_{l\in [K]}\widehat{\mu}_{l}^{(i)}(t-1)}\label{eq:noLF}
\end{align}

Taking the expectation of (\ref{eq:regdecLF})  and (\ref{eq:noLF}) we have
\begin{align}
\sum_{i\in V^{\prime}_{\gamma}}\left(\sum_{t=1}^T\P\left(A_t^{(i)}=k\right)+\sum_{j\in \mathcal{N}_{\gamma}^{(i)}}\sum_{t=1}^T\P\left(A_t^{(j)}=k\right)\right)\leq \Bar{\gamma}(G_{\gamma})\eta_k+(N-\Bar{\gamma}(G))(3\gamma-1)
\nonumber\\
+\sum_{i\in V^{\prime}_{\gamma}}\left(\Big |\mathcal{N}_{\gamma}^{(i)}\Big |+1\right)\sum_{t>\tau^{(i)}_{k}}^T \P\left(A^{(i)}_t = k\right)+ \sum_{i\in V^{\prime}_{\gamma}}\Big |\mathcal{N}_{\gamma}^{(i)}\Big |\sum_{t=1}^{\tau_{k}^{(i)}}\P\left(A_{t}^{(i)}=k,A_{t}^{(i)}= \argmax_{l\in [K]}\widehat{\mu}_{l}^{(i)}(t-1)\right)
\label{eq:suboptLF}
\end{align}

Now we proceed to upper bound last term of (\ref{eq:suboptLF}) as follows. Note that for any suboptimal arm $k$ we have,
\begin{align}
\P&\left(A_{t+1}^{(i)}=k,A_{t+1}^{(i)}= \argmax_{l\in [K]}\widehat{\mu}_{l}^{(i)}(t)\right) 
\nonumber\\\ 
&  \leq\P\left(\widehat{\mu}_k^{(i)}(t)+C_k^{(i)}(t)\geq \widehat{\mu}_1^{(i)}(t)+C_1^{(i)}(t), \widehat{\mu}_k^{(i)}(t)\geq \widehat{\mu}_1^{(i)}(t), \:\:
\widehat{\mu}_1^{(i)}(t)\leq \mu_{1}-C_1^{(i)}(t)\right)
\nonumber\\
& +\P\left(\widehat{\mu}_k^{(i)}(t)+C_k^{(i)}(t)\geq \widehat{\mu}_1^{(i)}(t)+C_1^{(i)}(t), \widehat{\mu}_k^{(i)}(t)\geq \widehat{\mu}_1^{(i)}(t), \:\:
\widehat{\mu}_1^{(i)}(t)> \mu_{1}-C_1^{(i)}(t)\right)
\nonumber\\\ 
&  \leq \P\left(\widehat{\mu}_1^{(i)}(t)\leq \mu_{1}-C_1^{(i)}(t)\right)
\nonumber\\\ 
&+\P\left(\widehat{\mu}_k^{(i)}(t)+C_k^{(i)}(t)\geq \widehat{\mu}_1^{(i)}(t)+C_1^{(i)}(t), \widehat{\mu}_k^{(i)}(t)\geq \widehat{\mu}_1^{(i)}(t), \:\:
\widehat{\mu}_1^{(i)}(t)> \mu_{1}-C_1^{(i)}(t)\right)\label{eq:condeqLF}
\end{align}
Now we proceed to upper bound the last term of (\ref{eq:condeqLF}) as follows. Note that we have
\begin{align}
\P& \left(\widehat{\mu}_k^{(i)}(t)+C_k^{(i)}(t)\geq \widehat{\mu}_1^{(i)}(t)+C_1^{(i)}(t), \widehat{\mu}_k^{(i)}(t)\geq \widehat{\mu}_1^{(i)}(t), \:\:
\widehat{\mu}_1^{(i)}(t)> \mu_{1}-C_1^{(i)}(t)\right)
\nonumber \\
& \leq\P\left(\widehat{\mu}_k^{(i)}(t)+C_k^{(i)}(t)>\mu_{1}-C_1^{(i)}(t)+C_k^{(i)}(t)\right)
\nonumber \\
& \leq \P\left(\widehat{\mu}_k^{(i)}(t)\geq \mu_k+C_k^{(i)}(t)\right).\label{eq:optsamLF}
\end{align}

From (\ref{eq:condeqLF}) and (\ref{eq:optsamLF}) we have
\begin{align}
\P\left(A_{t}^{(i)}=k,A_{t}^{(i)}= \argmax_{l\in [K]}\widehat{\mu}_{l}^{(i)}(t-1)\right)  \leq \P\left(\widehat{\mu}_1^{(i)}(t-1)\leq \mu_{1}-C_1^{(i)}(t-1)\right)
\nonumber\\\ +\P\left(\widehat{\mu}_k^{(i)}(t-1)\geq \mu_k+C_k^{(i)}(t-1)\right). \label{eq:addtailLF}
\end{align}

Recall that $
N_k^{(i)}(t)> \eta_k, \forall t> \tau^{(i)}_{k}.$  Thus from (\ref{eq:suboptLF}), (\ref{eq:addtailLF}) and Lemma \ref{lem:tailrestate} we have
\begin{align}
\sum_{i\in V^{\prime}_{\gamma}}\left(\sum_{t=1}^T\P\left(A_t^{(i)}=k\right)+\sum_{j\in \mathcal{N}_{\gamma}^{(i)}}\sum_{t=1}^T\P\left(A_t^{(j)}=k\right)\right)\leq \Bar{\gamma}(G_{\gamma})\eta_k+N+(N-\Bar{\gamma}(G))(3\gamma-1)
\nonumber\\
+\sum_{i\in V^{\prime}_{\gamma}}\left(\Big |\mathcal{N}_{\gamma}^{(i)}\Big |+1\right)\sum_{t=1}^{T-1} \left[\P \left(\widehat{\mu}_{1}^{(i)}(t)\leq \mu_{1}-C_{1}^{(i)}(t)\right)+\P \left(\widehat{\mu}_{k}^{(i)}(t)\geq \mu_{k}+C_{k}^{(i)}(t)\right)\right].
\label{eq:subsamplLF}
\end{align}

The proof of Lemma \ref{lem:RegDecompLF} follows from (\ref{eq:rearrangeLF}) and (\ref{eq:subsamplLF}). 
\end{proof}

Now we proceed to prove Theorem \ref{thm:regretLFUCB} as follows. We start by obtaining a modified tail bound similar to the result in Lemma \ref{lem:tailApp}. Note that $\forall i\in V^{\prime}_{\gamma}$ we have $1\leq N_k^{(i)}(t)< d_{\gamma}^{(i)}t.$ Thus considering $D_t=\frac{\log \left(\left(d_{\gamma}^{(i)}+1\right)t\right)}{\log \zeta}$ for any $\zeta>1$ in Lemma \ref{lem:tailApp} we get
\begin{align}
  \P\left(\Big|\widehat{\mu}_k^{(i)}(t)-{\mu}_k\Big|\geq \sigma\sqrt{\frac{2(\xi+1)\log t}{N_k^{(i)}(t)}}\right)\leq \frac{1}{\log \zeta}\frac{\log \left((d_{\gamma}^{(i)}+1)t\right)}{t^{(\xi+1)\left(1-\frac{(\zeta-1)^2}{16}\right)}}.\label{eq:modtailLF}
\end{align}
The proof of Theorem \ref{thm:regretLFUCB} follows from Lemmas \ref{lem:tailsum}, \ref{lem:RegDecompLF} and (\ref{eq:modtailLF}).

\section{Proof of Theorem \ref{thm:comcostLF}}\label{sec:comcostLFUCB}
Following a similar approach to the proof of Theorem \ref{thm:comcostMP} we obtain
\begin{align}
\sum_{i=1}^N\sum_{t=1}^{T-1}\P\left(A_{t}^{(i)}\neq 1,A_{t}^{(i)}\neq \argmax_{k\in [K]}\widehat{\mu}_{k}^{(i)}(t-1)\right)\leq \sum_{i=1}^N\sum_{t=1}^{T-1}\P\left(A_{t}^{(i)}\neq 1\right)\leq \sum_{i=1}^N\sum_{k=2}^K\expe \left[n_{k}^{(i)}(T)\right]. \label{eq:subOptLF}
\end{align}
Similarly we get 
\begin{align}
 \sum_{t=1}^{T-1}\P\left(A_{t}^{(i)}=1,A_{t}^{(i)}\neq \argmax_{k\in [K]}\widehat{\mu}_{k}^{(i)}(t-1)\right)& \leq   \frac{8\sigma(\xi+1)}{\Bar{\Delta}^2}\log T+2\sum_{t=1}^T\frac{1}{\log \zeta}\frac{\log \left((d_{\gamma}^{(i)}+1)t\right)}{t^{(\xi+1)\left(1-\frac{(\zeta-1)^2}{16}\right)}}\label{eq:pulloptLF}
\end{align}
From (\ref{eq:subOptLF}) and (\ref{eq:pulloptLF}) we have
\begin{align}
\sum_{i=1}^N\sum_{t=1}^{T-1}\P\left(A_{t}^{(i)}\neq \argmax_{k\in [K]}\widehat{\mu}_{k}^{(i)}(t-1)\right)\leq \sum_{i=1}^N\sum_{k=2}^K\expe \left[n_{k}^{(i)}(T)\right]
\nonumber\\
+\sum_{i=1}^N\frac{8\sigma(\xi+1)}{\Bar{\Delta}^2}\log T+2\sum_{i=1}^N\sum_{t=1}^{T-1}\frac{1}{\log \zeta}\frac{\log \left((d_{\gamma}^{(i)}+1)t\right)}{t^{(\xi+1)\left(1-\frac{(\zeta-1)^2}{16}\right)}}\label{eq:initiatedLF}
\end{align}
Note that (\ref{eq:initiatedLF}) is the expected number of messages initiated by all the agents. Recall that in ComEx-LFUCB a message initiated by agent $i$ is subsequently passed by agents within a $\gamma-1$ distance in graph $G.$ Thus we have
\begin{align}
\expe \left[L(T)\right]\leq \sum_{i=1}^N(d_{\gamma-1}^{(i)}+1)\sum_{t=1}^{T-1}\P\left(A_{t}^{(i)}\neq \argmax_{k\in [K]}\widehat{\mu}_{k}^{(i)}(t-1)\right)\label{eq:repeatmessageLF}
\end{align}
From (\ref{eq:initiatedLF}) and (\ref{eq:repeatmessageLF}) we have
\begin{align}
\expe \left[L(T)\right]\leq \sum_{i=1}^N(d_{\gamma-1}^{(i)}+1)\sum_{k=2}^K\expe \left[n_{k}^{(i)}(T)\right]
+\sum_{i=1}^N(d_{\gamma-1}^{(i)}+1)\frac{8\sigma(\xi+1)}{\Bar{\Delta}^2}\log T
\nonumber\\
+2\sum_{i=1}^N(d_{\gamma-1}^{(i)}+1)\sum_{t=1}^T\frac{1}{\log \zeta}\frac{\log \left((d_{\gamma}^{(i)}+1)t\right)}{t^{(\xi+1)\left(1-\frac{(\zeta-1)^2}{16}\right)}}\label{eq:comregLF}
\end{align}
From (\ref{eq:modtailLF}), (\ref{eq:comregLF}) and Lemma \ref{lem:RegDecompLF} we have
\begin{align}
\expe \left[L(T)\right]\leq \sum_{i=1}^N(d_{\gamma-1}^{(i)}+1)\sum_{k=2}^K \Bar{\gamma}(G_{\gamma})\eta_k+N+(N-\Bar{\gamma}(G_{\gamma}))(3\gamma-1)
\nonumber\\
+\sum_{i=1}^N(d_{\gamma-1}^{(i)}+1)\frac{8\sigma(\xi+1)}{\Bar{\Delta}^2}\log T
+2K\sum_{i=1}^N(d_{\gamma-1}^{(i)}+1)\sum_{t=1}^T\frac{1}{\log \zeta}\frac{\log \left((d_{\gamma}^{(i)}+1)t\right)}{t^{(\xi+1)\left(1-\frac{(\zeta-1)^2}{16}\right)}}\label{eq:regfinalLF}
\end{align}
Recall that $\eta_k=\frac{8\sigma(\xi+1)}{\Delta_k^2}\log T.$ Thus the proof of Theorem \ref{thm:comcostLF} follows from (\ref{eq:regfinalLF}) and Lemma \ref{lem:tailsum}.

\section{Regret Under Full Communication}\label{sec:regfullcom}
In this section we provide theoretical bounds for group regret of Full-UCB, Full-MPUCB and Full-LFUCB as follows.

\subsection{Group Regret for Full-UCB}
We start by proving a Lemma similar to Lemma \ref{lem:RegDecomp}. 
\begin{lemma}\label{lem:comExFull}
Let $\eta_k=\left(\frac{8(\xi+1)\sigma^2}{\Delta^2_k}\right)\log T.$ Let $\EuScript{C}$ be a non overlapping clique covering and $\Bar{\chi}(G)$ be the clique covering number of the graph $G.$ Let $\tau_{k,\mathcal{C}}$ be the maximum time step such that the total number of pulls from arm $k$ by agents in the clique $\mathcal{C}\in \EuScript{C}$ is at most $\eta_k.$ Define $\tau_k:=\min_{\mathcal{C}}\tau_{k,\mathcal{C}}.$  Then we have
\begin{align*}
    \sum_{i=1}^N\expe[n^{(i)}_k(T)]  \leq \Bar{\chi}(G) \eta_k+N+\sum_{i=1}^N\sum_{t>\tau_{k}}^{T-1}\left[\P \left(\widehat{\mu}_{1}^{(i)}(t)\leq \mu_{1}-C_{1}^{(i)}(t)\right)+\P \left(\widehat{\mu}_{k}^{(i)}(t)\geq \mu_{k}+C_{k}^{(i)}(t)\right)\right]
\end{align*}
\end{lemma}

\begin{proof}
Let $\EuScript{C}$ be a non overlapping clique covering of the graph $G.$ Then we have
\begin{align}
\sum_{i=1}^N\expe[n_k^{(i)}(T)]=\sum_{\mathcal{C}\in \EuScript{C}}\sum_{i\in \mathcal{C}}\sum_{t=1}^T\P\left(A_t^{(i)}=k\right)\label{eq:regFUll}
\end{align}
 Let $\tau_{k,\mathcal{C}}$ be the maximum time step such that the total number of pulls from arm $k$ by agents in the clique $\mathcal{C}$ is at most $\eta_k.$ This can be stated as $
\tau_{k,\mathcal{C}} := \max \left\{t\in [T] :  \sum_{i\in \mathcal{C}}\sum_{\tau=1}^t\indicate{A_{\tau}^{(i)}=k}\leq \eta_k\right\}.
$
Further for all $i\in \mathcal{C}$ we have $
N_k^{(i)}(t)> \eta_k, \forall t> \tau_{k,\mathcal{C}}.$ 
We analyse the expected number of times all agents pull suboptimal arm $k$ as follows.
\begin{align}
\sum_{\mathcal{C}\in \EuScript{C}}\sum_{i\in \mathcal{C}}\sum_{t=1}^T \indicate{A^{(i)}_t = k}  &= \sum_{\mathcal{C}\in \EuScript{C}}\sum_{i\in \mathcal{C}}\sum_{t=1}^{\tau_{k,\mathcal{C}}} \indicate{A^{(i)}_t = k}+  \sum_{\mathcal{C}\in \EuScript{C}}\sum_{i\in \mathcal{C}}\sum_{t>\tau_{k,\mathcal{C}}}^T \indicate{A^{(i)}_t = k, N_k^{(i)}(t-1) >\eta_k} \label{eq:regdecFull}
\end{align}
Taking the expectation of (\ref{eq:regdecFull}) we have
\begin{align}
\sum_{\mathcal{C}\in \EuScript{C}}\sum_{i\in \mathcal{C}}\sum_{t=1}^T \P\left(A^{(i)}_t = k\right)  &=   \sum_{\mathcal{C}\in \EuScript{C}}\sum_{i\in \mathcal{C}}\sum_{t=1}^{\tau_{k,\mathcal{C}}} \P\left(A^{(i)}_t = k\right)+  \sum_{\mathcal{C}\in \EuScript{C}}\sum_{i\in \mathcal{C}}\sum_{t>\tau_{k,\mathcal{C}}}^T \P\left(A^{(i)}_t = k,N_k^{(i)}(t-1) >\eta_k\right)
\nonumber\\
&\leq \Bar{\chi}(G)\eta_k+N+  \sum_{\mathcal{C}\in \EuScript{C}}\sum_{i\in \mathcal{C}}\sum_{t>\tau_{k,\mathcal{C}}}^{T-1} \P\left(A^{(i)}_{t+1} = k,N_k^{(i)}(t) >\eta_k\right)
\label{eq:regDecComFull}
\end{align}
Let $\tau_k:=\min_{\mathcal{C}}\tau_{k,\mathcal{C}\in\EuScript{C}}.$ Similarly to Lemma \ref{lem:RegDecomp} from (\ref{eq:regFUll}), (\ref{eq:regDecComFull}) and Lemma \ref{lem:tailrestate} we have
\begin{align*}
    \sum_{i=1}^N\expe[n^{(i)}_k(T)]  \leq \Bar{\chi}(G) \eta_k+N+\sum_{i=1}^N\sum_{t>\tau_{k}}^{T-1}\left[\P \left(\widehat{\mu}_{1}^{(i)}(t)\leq \mu_{1}-C_{1}^{(i)}(t)\right)+\P \left(\widehat{\mu}_{k}^{(i)}(t)\geq \mu_{k}+C_{k}^{(i)}(t)\right)\right]
\end{align*}
This concludes the proof of Lemma  \ref{lem:comExFull}.
\end{proof}
Then from Lemmas \ref{lem:tailApp}, \ref{lem:tailsum} and \ref{lem:comExFull} it follows that
\begin{align*}
     \expe \left[R(T)\right]= O\left(K\Bar{\chi}(G)\log T+KN\right).
\end{align*}


\subsection{Group Regret for Full-MPUCB}
We start by proving a Lemma similar to Lemma \ref{lem:RegDecompMP}. 

\begin{lemma}\label{lem:comExFullMP}
Let $\eta_k=\left(\frac{8(\xi+1)\sigma^2}{\Delta^2_k}\right)\log T.$ Let $\EuScript{C}_{\gamma}$ be a non overlapping clique covering and $\Bar{\chi}(G_{\gamma})$ be the clique covering number of the graph $G_{\gamma},$ which is the $\gamma^{\mathrm{th}}$ power graph of $G.$ Let $\tau_{k,\mathcal{C}}$ be the maximum time step such that the total number of pulls from arm $k$ by agents in the clique $\mathcal{C}\in \EuScript{C}_{\gamma}$ is at most $\eta_k+(|\mathcal{C}-1|)(\gamma-1).$ Define $\tau_k:=\min_{\mathcal{C}}\tau_{k,\mathcal{C}}.$  Then we have
\begin{align*}
    \sum_{i=1}^N\expe[n^{(i)}_k(T)]  \leq \Bar{\chi}(G_{\gamma}) \eta_k+N+(N-\Bar{\chi}(G_{\gamma}))(\gamma-1)
    \\
    +\sum_{i=1}^N\sum_{t>\tau_{k}}^{T-1}\left[\P \left(\widehat{\mu}_{1}^{(i)}(t)\leq \mu_{1}-C_{1}^{(i)}(t)\right)+\P \left(\widehat{\mu}_{k}^{(i)}(t)\geq \mu_{k}+C_{k}^{(i)}(t)\right)\right]
\end{align*}
\end{lemma}

\begin{proof}
Let $\EuScript{C}_{\gamma}$ be a non overlapping clique covering of the graph $G_{\gamma}.$ Then we have
\begin{align}
\sum_{i=1}^N\expe[n_k^{(i)}(T)]=\sum_{\mathcal{C}\in \EuScript{C}_{\gamma}}\sum_{i\in \mathcal{C}}\sum_{t=1}^T\P\left(A_t^{(i)}=k\right)\label{eq:regFUllMP}
\end{align}
 Let $\tau_{k,\mathcal{C}_{\gamma}}$ be the maximum time step such that the total number of pulls from arm $k$ by agents in the clique $\mathcal{C}$ is at most $\eta_k.$ This can be stated as $
\tau_{k,\mathcal{C}} := \max \left\{t\in [T] :  \sum_{i\in \mathcal{C}}\sum_{\tau=1}^t\indicate{A_{\tau}^{(i)}=k}\leq \eta_k+(|\mathcal{C}|-1)(\gamma-1)\right\}.
$
Further for all $i\in \mathcal{C}$ we have $
N_k^{(i)}(t)> \eta_k, \forall t> \tau_{k,\mathcal{C}}.$ 
We analyse the expected number of times all agents pull suboptimal arm $k$ as follows.
\begin{align}
\sum_{\mathcal{C}\in \EuScript{C}_{\gamma}}\sum_{i\in \mathcal{C}}\sum_{t=1}^T \indicate{A^{(i)}_t = k}  &= \sum_{\mathcal{C}\in \EuScript{C}_{\gamma}}\sum_{i\in \mathcal{C}}\sum_{t=1}^{\tau_{k,\mathcal{C}}} \indicate{A^{(i)}_t = k}+  \sum_{\mathcal{C}\in \EuScript{C}_{\gamma}}\sum_{i\in \mathcal{C}}\sum_{t>\tau_{k,\mathcal{C}}}^T \indicate{A^{(i)}_t = k, N_k^{(i)}(t-1) >\eta_k} \label{eq:regdecFullMP}
\end{align}
Taking the expectation of (\ref{eq:regdecFullMP}) we have
\begin{align}
\sum_{\mathcal{C}\in \EuScript{C}_{\gamma}}\sum_{i\in \mathcal{C}}\sum_{t=1}^T \P\left(A^{(i)}_t = k\right)  &=   \sum_{\mathcal{C}\in \EuScript{C}_{\gamma}}\sum_{i\in \mathcal{C}}\sum_{t=1}^{\tau_{k,\mathcal{C}}} \P\left(A^{(i)}_t = k\right)+  \sum_{\mathcal{C}\in \EuScript{C}_{\gamma}}\sum_{i\in \mathcal{C}}\sum_{t>\tau_{k,\mathcal{C}}}^T \P\left(A^{(i)}_t = k,N_k^{(i)}(t-1) >\eta_k\right)
\nonumber\\
&\leq \Bar{\chi}(G_{\gamma}) \eta_k+N+(N-\Bar{\chi}(G_{\gamma}))(\gamma-1)\\
&+ \sum_{\mathcal{C}\in \EuScript{C}_{\gamma}}\sum_{i\in \mathcal{C}}\sum_{t>\tau_{k,\mathcal{C}}}^{T-1} \P\left(A^{(i)}_{t+1} = k,N_k^{(i)}(t) >\eta_k\right)
\label{eq:regDecComFullMP}
\end{align}
Let $\tau_k:=\min_{\mathcal{C}\in\EuScript{C}_{\gamma}}\tau_{k,\mathcal{C}}.$ Similarly to Lemma \ref{lem:RegDecompMP} from (\ref{eq:regFUllMP}), (\ref{eq:regDecComFullMP}) and Lemma \ref{lem:tailrestate} we have
\begin{align*}
    \sum_{i=1}^N\expe[n^{(i)}_k(T)]  \leq \Bar{\chi}(G_{\gamma}) \eta_k+N+(N-\Bar{\chi}(G_{\gamma}))(\gamma-1)\\
    +\sum_{i=1}^N\sum_{t>\tau_{k}}^{T-1}\left[\P \left(\widehat{\mu}_{1}^{(i)}(t)\leq \mu_{1}-C_{1}^{(i)}(t)\right)+\P \left(\widehat{\mu}_{k}^{(i)}(t)\geq \mu_{k}+C_{k}^{(i)}(t)\right)\right]
\end{align*}
This concludes the proof of Lemma  \ref{lem:comExFullMP}.
\end{proof}
Then from Lemmas \ref{lem:tailApp}, \ref{lem:tailsum} and \ref{lem:comExFullMP} it follows that
\begin{align*}
     \expe \left[R(T)\right]= O\left(K\Bar{\chi}(G_{\gamma})\log T+KN\right).
\end{align*}


\subsection{Group Regret for Full-LFUCB}
We begin the proof by providing a lemma similar to Lemma \ref{lem:RegDecompLF}.

\begin{lemma}\label{lem:RegDecompLFFull}
Let $\Bar{\gamma}(G_{\gamma})$ is the clique covering number of graph $G_{\gamma}.$ Let $\eta_k=\left(\frac{8(\xi+1)\sigma^2}{\Delta^2_k}\right)\log T.$ Then we have
\begin{align*}
    \sum_{i=1}^N\expe[n^{(i)}_k(T)]  & \leq  \Bar{\gamma}(G_{\gamma})\eta_k+N+(N-\Bar{\gamma}(G_{\gamma}))(3\gamma-1)
    \\
    & +\sum_{i\in V^{\prime}_{\gamma}}\left(\Big |\mathcal{N}_{\gamma}^{(i)}\Big |+1\right)\sum_{t>\tau_k^{(i)}}^{T-1}\left[\P \left(\widehat{\mu}_{1}^{(i)}(t)\leq \mu_{1}-C_{1}^{(i)}(t)\right)+\P \left(\widehat{\mu}_{k}^{(i)}(t)\geq \mu_{k}+C_{k}^{(i)}(t)\right)\right]
\end{align*}
where $V^{\prime}_{\gamma}$ is the maximal dominating set of $G_{\gamma}$ and $\mathcal{N}_{\gamma}^{(i)}$ is the set of followers of leader $i.$ Here ${\tau}^{(i)}_{k}$ be the maximum time step such that the total number of times agent $i$ pulls arm $k$ and the number of times agents in $\mathcal{N}_{\gamma}^{(i)}$ pull from arm $k$ is at most $\eta_k+\mathcal{N}_{\gamma}^{(i)}(\gamma-1).$ 
\end{lemma}
\begin{proof}
Recall that $V^{\prime}_{\gamma}$ is the maximal dominating set of $G_{\gamma}.$ Let $\mathcal{N}_{\gamma}^{(i)}$ be the set of followers of leader $i.$ Then for each suboptimal arm $k > 1$ we have
\begin{align}
    \sum_{i=1}^N\expe[n^{(i)}_k(T)]
    & = \sum_{i\in V^{\prime}_{\gamma}}\left(\sum_{t=1}^T \P\left(A^{(i)}_t = k\right)+\sum_{j\in \mathcal{N}_{\gamma}^{(i)} }\sum_{t=1}^T \P\left(A^{(j)}_t = k\right)\right) \label{eq:rearrangeLFFull}
\end{align}

 Let ${\tau}^{(i)}_{k}$ be the maximum time step such that the total number of times agent $i$ pulls arm $k$ and the number of times agents in $\mathcal{N}_{\gamma}^{(i)}$ pull from arm $k$ is at most $\eta_k+\mathcal{N}_{\gamma}^{(i)}(\gamma-1).$ This can be stated as \begin{align*}
\tau^{(i)}_{k} := \max \left\{t\in [T] :  \sum_{\tau=1}^{t}\indicate{A_{\tau}^{(i)}=k}+\sum_{j\in \mathcal{N}_{\gamma}^{(i)}}\sum_{\tau=1}^t\indicate{A_{\tau}^{(i)}=k}\leq \eta_k+\mathcal{N}_{\gamma}^{(i)}(\gamma-1)\right\}.
\end{align*} 
Then we have $
N_k^{(i)}(t)> \eta_k, \forall t> \tau^{(i)}_{k}.$ 
We analyse the expected number of times all agents pull suboptimal arm $k$ as follows. Let $d(i,j)$ be the distance between agents $i$ and $j$ in graph $G.$ Then note that for any $j\in \mathcal{N}_{\gamma}^{(i)}$ we have $A_t^{(j)}=A_{t-d(i,j)}^{(i)}$ and $d(i,j)\leq \gamma.$ 
\begin{align}
\sum_{i\in V^{\prime}_{\gamma}}\left\{\sum_{t=1}^T \indicate{A^{(i)}_t = k}+\sum_{j\in \mathcal{N}_{\gamma}^{(i)} }\sum_{t=1}^T \indicate{A^{(j)}_t = k}\right\}  \leq \sum_{i\in V^{\prime}_{\gamma}}\left\{\sum_{t=1}^{\tau^{(i)}_{k}} \indicate{A^{(i)}_t = k}
\right.
\nonumber\\
\left.
+\sum_{j\in \mathcal{N}_{\gamma}^{(i)} }\sum_{t=d(i,j)}^{\tau^{(i)}_{k}} \indicate{A^{(j)}_t = k}\right\}
\nonumber\\
+\sum_{i\in V^{\prime}_{\gamma}}\left\{\sum_{t>\tau^{(i)}_{k}}^T \indicate{A^{(i)}_t = k}+\sum_{j\in \mathcal{N}_{\gamma}^{(i)} }\sum_{t>\tau^{(i)}_{k}}^{T-d(i,j)} \indicate{A^{(i)}_t = k}\right\} +\sum_{i\in V^{\prime}_{\gamma}}\sum_{j\in \mathcal{N}_{\gamma}^{(i)}}2d(i,j)
\nonumber\\
\leq\sum_{i\in V^{\prime}_{\gamma}}\left\{\sum_{t=1}^{\tau^{(i)}_{k}} \indicate{A^{(i)}_t = k}+\sum_{j\in \mathcal{N}_{\gamma}^{(i)} }\sum_{t=d(i,j)}^{\tau^{(i)}_{k}} \indicate{A^{(j)}_t = k}\right\}
\nonumber\\
+\sum_{i\in V^{\prime}_{\gamma}}\left(\Big |\mathcal{N}_{\gamma}^{(i)}\Big |+1\right)\sum_{t>\tau^{(i)}_{k}}^T \indicate{A^{(i)}_t = k}+2(N -\Bar{\gamma}(G_{\gamma}))\gamma
\label{eq:regdecLFFull}
\end{align}
Now we proceed to upper bound the first two terms of right hand side of (\ref{eq:regdecLFFull}) as follows. Note that we have
\begin{align}
\sum_{i\in V^{\prime}_{\gamma}}\left\{\sum_{t=1}^{\tau^{(i)}_{k}} \indicate{A^{(i)}_t = k}+\sum_{j\in \mathcal{N}_{\gamma}^{(i)} }\sum_{t=d(i,j)}^{\tau^{(i)}_{k}} \indicate{A^{(j)}_t = k}\right\} 
\nonumber\\
\leq \sum_{i\in V^{\prime}_{\gamma}}\left(\eta_k+\mathcal{N}_{\gamma}^{(i)}(\gamma-1)\right)+ \sum_{i\in V^{\prime}_{\gamma}}\Big |\mathcal{N}_{\gamma}^{(i)}\Big |\sum_{t=1}^{\tau_{k}^{(i)}}\indicate{A_{t}^{(i)}=k}\label{eq:noLFFull}
\end{align}

Taking the expectation of (\ref{eq:regdecLFFull})  and (\ref{eq:noLFFull}) we have
\begin{align}
\sum_{i\in V^{\prime}_{\gamma}}\left(\sum_{t=1}^T\P\left(A_t^{(i)}=k\right)+\sum_{j\in \mathcal{N}_{\gamma}^{(i)}}\sum_{t=1}^T\P\left(A_t^{(j)}=k\right)\right)\leq \Bar{\gamma}(G_{\gamma})\eta_k+N+(N-\Bar{\gamma}(G_{\gamma}))(3\gamma-1)
\nonumber\\
+\sum_{i\in V^{\prime}_{\gamma}}\left(\Big |\mathcal{N}_{\gamma}^{(i)}\Big |+1\right)\sum_{t>\tau^{(i)}_{k}}^T \P\left(A^{(i)}_t = k\right)+ \sum_{i\in V^{\prime}_{\gamma}}\Big |\mathcal{N}_{\gamma}^{(i)}\Big |\sum_{t>\tau^{(i)}_{k}}^{\tau_{k}^{(i)}}\P\left(A_{t}^{(i)}=k\right)
\label{eq:suboptLFFull}
\end{align}
Recall that $
N_k^{(i)}(t)> \eta_k, \forall t> \tau^{(i)}_{k}.$  Thus from (\ref{eq:suboptLFFull}) and Lemma \ref{lem:tailrestate} we have
\begin{align}
\sum_{i\in V^{\prime}_{\gamma}}\left(\sum_{t=1}^T\P\left(A_t^{(i)}=k\right)+\sum_{j\in \mathcal{N}_{\gamma}^{(i)}}\sum_{t=1}^T\P\left(A_t^{(j)}=k\right)\right)\leq \Bar{\gamma}(G_{\gamma})\eta_k+(N-\Bar{\gamma}(G_{\gamma}))(3\gamma-1)
\nonumber\\
+\sum_{i\in V^{\prime}_{\gamma}}\left(\Big |\mathcal{N}_{\gamma}^{(i)}\Big |+1\right)\sum_{t>\tau_k^{(i)}}^{T-1} \left[\P \left(\widehat{\mu}_{1}^{(i)}(t)\leq \mu_{1}-C_{1}^{(i)}(t)\right)+\P \left(\widehat{\mu}_{k}^{(i)}(t)\geq \mu_{k}+C_{k}^{(i)}(t)\right)\right].
\label{eq:subsamplLFFull}
\end{align}

The proof of Lemma \ref{lem:RegDecompLFFull} follows from (\ref{eq:rearrangeLFFull}) and (\ref{eq:subsamplLFFull}). 
\end{proof}

Then from Lemmas \ref{lem:tailApp}, \ref{lem:tailsum} and \ref{lem:RegDecompLFFull} it follows that
\begin{align*}
     \expe \left[R(T)\right]= O\left(K\Bar{\gamma}(G_{\gamma})\log T+KN\right).
\end{align*}


\section{Additional Experimental Results}\label{sec:simulation}
In this section we provide additional simulation results. We observe that performance of the algorithms improve when we decrease $\xi.$ Thus for simulations provided in this section we use $\xi=1.001.$ Further when $\gamma$ is increased communication density increases and performance improve. For simulations provided in this section we consider $\gamma = 7.$ We use the same graph structure and reward structures used in the results provided in the main paper.

\begin{figure*}[h]
    \centering
    \includegraphics[width=0.95\textwidth]{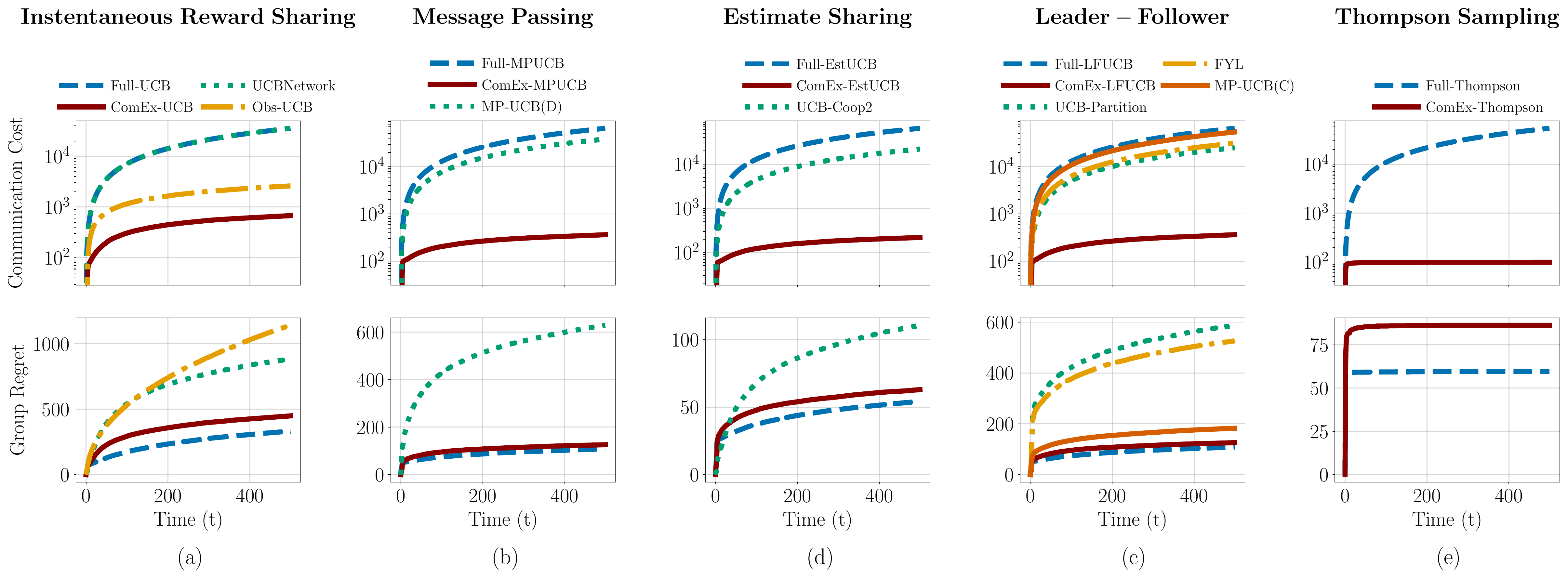}
    \caption{\small{ A comparison of expected cumulative group regret and communication cost of our algorithms and existing state-of-the-art algorithms in several benchmark cooperative bandit frameworks.}}
    \label{fig:performanceApp}
\end{figure*}

\paragraph{Additional details on estimate sharing}
Note that in estimate sharing agents average their estimates of instantaneously suboptimal arms at every time step. Thus at each time step each agent creates $2K$ number of messages (estimated sum of rewards for each arm and estimated number of pulls from each arm). If the number of arms are of same order as time horizon this leads to $O(T^2)$ cost for Full-EstUCB and $O(T\log T)$ cost for ComEx-EstUCB. However we consider that nummber of arms are fixed and $K<< T$ for large $T$ and when providing simulation results for the communication cost we only considered the communication cost associated with initiating messages and passing them through network neglecting the dependence on number of arms. This leads to $O(T)$ cost for Full-EstUCB and $O(\log T)$ cost for ComEx-EstUCB.

\newpage
\section{Pseudo code of ComEx-UCB}\label{sec:algUCB}
\vspace{1em}

\begin{algorithm}[ht!] 
\caption{ComEx-UCB}\label{alg:ComEx-UCB}
    \SetAlgoLined
    \KwIn{Arms $k\in [K],$ variance proxy upper bound $\sigma^2$, parameter $\xi$}
    {\bf  Initialize:} {$ N_k^{(i)}(0)= \widehat{\mu}_k^{(i)}(0)=C_k^{(i)}(0)=0, \forall k,i$}
    
    \For{\text{\normalfont\ each iteration} $t \in [T]$}{
        $E_t\gets \emptyset$
            
         \For{\text{\normalfont\ each agent} $i \in [N]$}{
           \tcc{Sampling phase}
           \If{$t=1$}{
               $A_{t}^{(i)}\gets \textsc{RandomArm} \left([K]\right)$
           }\Else{
            
             $A_{t}^{(i)}\gets\argmax_{k}\widehat{\mu}_k^{(i)}(t-1)+C_k^{(i)}(t-1)$
            }
             \tcc{Send messages}
            \If {$A_{t}^{(i)}\neq \argmax_k{\widehat{\mu}_k^{(i)}(t-1)}$}{ 
             $\textsc{Send}\left(m_t^{(i)}:=\Big \langle A_t^{(i)}, X_t^{(i)}\Big \rangle\right)$
                
            }
            $\mathbf{m}_t^{(i)}\gets m_t^{(i)}$
            
        }
        \For {\text{\normalfont\ each agent} $i \in [N]$}{
            \tcc{Receive messages}
            \For{\text{\normalfont\ each neighbor} $j$ \text{\normalfont\ s. t.} $\{(j\to i)\in E\}$}{$\mathbf{m}_t^{(i)}\gets \mathbf{m}_t^{(i)}\cup m_t^{(j)}$
            }
            
             \For{\text{\normalfont\ each agent $j\in [N]$}}{
             \If{$m_t^{(j)}\in \mathbf{m}_t^{(i)}$}{
                $E_t\gets E_t \cup \{(j\to i)$\}
             }
             }
            \tcc{Update estimates}
            \For {\text{\normalfont\ each arm} $k \in [K]$}{
            
                \textsc{Calculate}   $\left(N_k^{(i)}(t), \widehat{\mu}_k^{(i)}(t),C_k^{(i)}(t)\right)$ 
            }
        }
    }
\end{algorithm}

\section{Pseudo code of ComEx-MPUCB}\label{sec:algMPUCB}
\vspace{1em}

\begin{algorithm}[ht!]
    \SetAlgoLined
    \KwIn{Arms $k\in [K],$ variance proxy upper bound $\sigma_k^2$, parameter $\xi,\gamma$}
    {\bf  Initialize:} {$ N_k^{(i)}(0)= \widehat{\mu}_k^{(i)}(0)=C_k^{(i)}(0)=0, \forall k,i$}
    
    \For{\text{\normalfont\ each iteration} $t \in [T]$}{
        $E_t\gets \emptyset$
            
         \For{\text{\normalfont\ each agent} $i \in [N]$}{
           \tcc{Sampling phase}
           \If{$t=1$}{
               $A_{t}^{(i)}\gets \textsc{RandomArm} \left([K]\right)$
           }\Else{
            
             $A_{t}^{(i)}\gets\argmax_{k}\widehat{\mu}_k^{(i)}(t-1)+C_k^{(i)}(t-1)$
            }
             \tcc{Send messages}
             \If {$A_{t}^{(i)}\neq \argmax_k{\widehat{\mu}_k^{(i)}(t-1)}$}{  $\textsc{Create}\left(m_t^{(i)}:=\Big \langle i, t, A_t^{(i)}, X_t^{(i)}\Big \rangle\right)$
             
             $\mathbf{m}_t^{(i)}\gets m_t^{(i)}$
            }
            
            $\textsc{Send}\left(\mathbf{m}_t^{(i)}\right)$
        }
        \For {\text{\normalfont\ each agent} $i \in [N]$}{
            \tcc{Receive messages}
            \For{\text{\normalfont\ each neighbor} $j$ \text{\normalfont\ s. t.} $\{(j\to i)\in E\}$}{$\mathbf{m}_t^{(i)}\gets \mathbf{m}_t^{(i)}\cup m_t^{(j)}$
            }
            
            \tcc{Discard messages older than $\gamma$}
            \For{\text{\normalfont each neighbor} $j\in [N]$ }{$\mathbf{m}_t^{(i)}\gets \mathbf{m}_{t}^{(i)}\symbol{92} m_{\tau}^{(j)}, \forall \tau$ \text{\normalfont\ s. t.} $\tau < t -\gamma$
            
            \For{\text{\normalfont each time step} $\tau\in \{t-\gamma+1, \ldots, t\}$}{
                \If{$m_{\tau}^{(j)}\in \mathbf{m}_t^{(i)}$}{
                    $E_{\tau}\gets E_{\tau} \cup \{(j\to i)\}$
                }
            }
            
             }
            \tcc{Update estimates}
            \For {\text{\normalfont\ each arm} $k \in [K]$}{
            
                \textsc{Calculate}   $\left(N_k^{(i)}(t), \widehat{\mu}_k^{(i)}(t),C_k^{(i)}(t)\right)$ 
            }
         $\mathbf{m}_{t+1}^{(i)}\gets \mathbf{m}_{t}^{(i)}$
        }
    }
\caption{ComEx-MPUCB}
\label{alg:ComEx-MPUCB}
\end{algorithm}


\section{Pseudo code of ComEx-LFUCB}\label{sec:algLFUCB}
\vspace{1em}
For all $i\in V^{\prime}_{\gamma}$ the indicator variable $I_t^{(i)}$ takes value 1 if $A_t^{(i)}$ is instantaneously suboptimal.

\begin{algorithm}[ht!]
    \SetAlgoLined
    \KwIn{Arms $k\in [K],$ variance proxy upper bound $\sigma_k^2$, parameter $\xi,\gamma$}
    {\bf  Initialize:} {$ N_k^{(i)}(0)= \widehat{\mu}_k^{(i)}(0)=C_k^{(i)}(0)=0, \forall k,i$}
    
    \For{\text{\normalfont\ each iteration} $t \in [T]$}{
        $E_t\gets \emptyset$
            
         \For{\text{\normalfont\ each agent} $i \in V^{\prime}_{\gamma}$}{
           \tcc{Sampling phase}
           Same as ComEx-MPUCB
           
             \tcc{Send messages}
             $\textsc{Create}\left(m_t^{(i)}:=\Big \langle i, t, A_t^{(i)}, I_t^{(i)}\Big \rangle\right)$

            $\mathbf{m}_t^{(i)}\gets m_t^{(i)}$
            
            $\textsc{Send}\left(\mathbf{m}_t^{(i)}\right)$
        
        \For{\text{\normalfont\ each agent} $j \in \mathcal{N}^{(i)}_{\gamma}$}{
        \tcc{Sampling phase}
           \If{$t<d(i,j)$}{
               $A_{t}^{(i)}\gets \textsc{RandomArm} \left([K]\right)$
           }\Else{
            
             $A_{t}^{(j)}\gets A_{t-d(i,j)}^{(i)}$
            }
             \If {$I_{t-d(i,j)}^{(i)}=1$}{  $\textsc{Create}\left(m_t^{(j)}:=\Big \langle j, t, A_t^{(j)}, X_t^{(j)}\Big \rangle\right)$
             
             $\mathbf{m}_t^{(j)}\gets m_t^{(j)}$
            }
        }
        }
        
        \For {\text{\normalfont\ each agent} $i \in V^{\prime}_{\gamma}$}{
            \tcc{Receive messages}
            \For{\text{\normalfont\ each neighbor} $j$ \text{\normalfont\ s. t.} $\{(j\to i)\in E\}$}{$\mathbf{m}_t^{(i)}\gets \mathbf{m}_t^{(i)}\cup m_t^{(j)}$
            }
            
            \tcc{Discard messages older than $\gamma$}
            \For{\text{\normalfont each neighbor} $j\in [N]$ }{$\mathbf{m}_t^{(i)}\gets \mathbf{m}_{t}^{(i)}\symbol{92} m_{\tau}^{(j)}, \forall \tau$ \text{\normalfont\ s. t.} $\tau < t -\gamma$
            
            \For{\text{\normalfont each time step} $\tau\in \{t-\gamma+1, \ldots, t\}$}{
                \If{$m_{\tau}^{(j)}\in \mathbf{m}_t^{(i)}$}{
                    $E_{\tau}\gets E_{\tau} \cup \{(j\to i)\}$
                }
            }
            
             }
            \tcc{Update estimates}
            \For {\text{\normalfont\ each arm} $k \in [K]$}{
            
                \textsc{Calculate}   $\left(N_k^{(i)}(t), \widehat{\mu}_k^{(i)}(t),C_k^{(i)}(t)\right)$ 
            }
         $\mathbf{m}_{t+1}^{(i)}\gets \mathbf{m}_{t}^{(i)}$
        }
    }
\caption{ComEx-LFUCB}
\label{alg:ComEx-LFUCB}
\end{algorithm}


\section{Pseudo code of ComEx-EstUCB}\label{sec:algEstUCB}
\vspace{1em}
Let $\widehat{N}_k^{(i)}(t)$ be the estimated number of pulls from arm $k$ for agent $i$ up to time $t.$
\begin{algorithm}[ht!]
    \SetAlgoLined
    \KwIn{Arms $k\in [K],$ variance proxy upper bound $\sigma_k^2$, parameter $\xi,\gamma$}
    {\bf  Initialize:} {$ \widehat{N}_k^{(i)}(0)= \widehat{\mu}_k^{(i)}(0)=C_k^{(i)}(0)=0, \forall k,i$}
    
    \For{\text{\normalfont\ each iteration} $t \in [T]$}{
        $E_t\gets \emptyset$
            
         \For{\text{\normalfont\ each agent} $i \in [N]$}{
           \tcc{Sampling phase}
           \If{$t=1$}{
               $A_{t}^{(i)}\gets \textsc{RandomArm} \left([K]\right)$
           }\Else{
            
             $A_{t}^{(i)}\gets\argmax_{k}\widehat{\mu}_k^{(i)}(t-1)+C_k^{(i)}(t-1)$
            }
             \tcc{Send messages}
             \If {$A_{t}^{(i)}\neq \argmax_k{\widehat{\mu}_k^{(i)}(t-1)}$}{  $\textsc{Create}\left(m_t^{(i)}:=\Big \langle i, t, \widehat{N}_t^{(i)},\widehat{\mu}_k^{(i)}(t-1)\Big \rangle\right)$
             
             $\mathbf{m}_t^{(i)}\gets m_t^{(i)}$
            }
            
            $\textsc{Send}\left(\mathbf{m}_t^{(i)}\right)$
        }
        \For {\text{\normalfont\ each agent} $i \in [N]$}{
            \tcc{Receive messages}
            \For{\text{\normalfont\ each neighbor} $j$ \text{\normalfont\ s. t.} $\{(j\to i)\in E\}$}{$\mathbf{m}_t^{(i)}\gets \mathbf{m}_t^{(i)}\cup m_t^{(j)}$
            }
            
            \tcc{Discard messages older than $\gamma$}
            \For{\text{\normalfont each neighbor} $j\in [N]$ }{$\mathbf{m}_t^{(i)}\gets \mathbf{m}_{t}^{(i)}\symbol{92} m_{\tau}^{(j)}, \forall \tau$ \text{\normalfont\ s. t.} $\tau < t -\gamma$
            
            \For{\text{\normalfont each time step} $\tau\in \{t-\gamma+1, \ldots, t\}$}{
                \If{$m_{\tau}^{(j)}\in \mathbf{m}_t^{(i)}$}{
                    $E_{\tau}\gets E_{\tau} \cup \{(j\to i)\}$
                }
            }
            
             }
            \tcc{Update estimates}
            \For {\text{\normalfont\ each arm} $k \in [K]$}{
            
                \textsc{Calculate}   $\left(\widehat{N}_k^{(i)}(t), \widehat{\mu}_k^{(i)}(t),C_k^{(i)}(t)\right)$ {\text{\normalfont according to consensus algorithm}} 
            }
         $\mathbf{m}_{t+1}^{(i)}\gets \mathbf{m}_{t}^{(i)}$
        }
    }
\caption{ComEx-EstUCB}
\label{alg:ComEx-EstUCB}
\end{algorithm}


\section{Pseudo code of ComEx-MPThompson}\label{sec:algMPTHmp}
\vspace{1em}
In Thompson sampling for each arm $k$ each agent $i$ maintains a posterior distribution $\phi_k^{(i)}$ and updates the distribution according to the available information. Then draw samples from the posterior distribution and pull the arm with highest sample value.

\begin{algorithm}[ht!]
    \SetAlgoLined
    \KwIn{Arms $k\in [K],$ parameter $\gamma$}
    {\bf  Initialize:} {$ \phi_k^{(i)}(0), \forall k,i$}
    
    \For{\text{\normalfont\ each iteration} $t \in [T]$}{
        $E_t\gets \emptyset$
            
         \For{\text{\normalfont\ each agent} $i \in [N]$}{
           \tcc{Sampling phase}
           \For{\text{\normalfont\ each arm} $k \in [K]$}{
            $y_k^{(i)}(t)\sim \phi_k^{(i)}(t-1)$
            
            $A_{t}^{(i)}\gets\argmax_{k}y_k^{(i)}(t)$
            }
             \tcc{Send messages}
             $\textsc{Create}\left(m_t^{(i)}:=\Big \langle i, t, A_t^{(i)}, X_t^{(i)}\Big \rangle\right)$

            $\mathbf{m}_t^{(i)}\gets m_t^{(i)}$
            
            $\textsc{Send}\left(\mathbf{m}_t^{(i)}\right)$
        }
        \For {\text{\normalfont\ each agent} $i \in [N]$}{
            \tcc{Receive messages}
            \For{\text{\normalfont\ each neighbor} $j$ \text{\normalfont\ s. t.} $\{(j\to i)\in E\}$}{$\mathbf{m}_t^{(i)}\gets \mathbf{m}_t^{(i)}\cup m_t^{(j)}$
            }
            
            \tcc{Discard messages older than $\gamma$}
            \For{\text{\normalfont each neighbor} $j\in [N]$ }{$\mathbf{m}_t^{(i)}\gets \mathbf{m}_{t}^{(i)}\symbol{92} m_{\tau}^{(j)}, \forall \tau$ \text{\normalfont\ s. t.} $\tau < t -\gamma$
            
            \For{\text{\normalfont each time step} $\tau\in \{t-\gamma+1, \ldots, t\}$}{
                \If{$m_{\tau}^{(j)}\in \mathbf{m}_t^{(i)}$}{
                    $E_{\tau}\gets E_{\tau} \cup \{(j\to i)\}$
                }
            }
            
             }
            \tcc{Update estimates}
            \For {\text{\normalfont\ each arm} $k \in [K]$}{
            
                \textsc{Calculate}   $\left( \phi_k^{(i)}(t)\right)$ 
            }
         $\mathbf{m}_{t+1}^{(i)}\gets \mathbf{m}_{t}^{(i)}$
        }
    }
\caption{ComEx-MPThompson}
\label{alg:ComEx-MPThompson}
\end{algorithm}

\end{document}